\documentclass[10pt,letterpaper,twocolumn]{article}       

\usepackage[utf8]{inputenc}
\usepackage[english]{babel}

\usepackage{mathptmx}      
\usepackage{natbib}
\usepackage{graphicx}
\usepackage{amsmath}
\usepackage{bm}
\usepackage{bbm}
\usepackage{algorithmic}
\usepackage{algorithm}
\usepackage{subfigure}
\usepackage{amsthm}
\newtheorem{proposition}{Proposition}
\newtheorem{assumption}{Assumption}

\usepackage[usenames,dvipsnames]{xcolor}
\definecolor{shadecolor}{gray}{0.9}

\usepackage[colorlinks,linktoc=all]{hyperref}
\hypersetup{citecolor=MidnightBlue}
\hypersetup{linkcolor=MidnightBlue}
\hypersetup{urlcolor=MidnightBlue}
\usepackage{url}

\newcommand{\bfv}{\mathbf{v}}

\newcommand{\comm}[1]{}

\begin{document}
\emergencystretch 3em

\title{Sequential Changepoint Detection in Neural Networks with Checkpoints
}


\date{}

\author{Michalis K. Titsias  \\
 DeepMind \\ 
   mtitsias@google.com 
  \and
        Jakub Sygnowski  \\
         DeepMind  \\
          sygi@google.com
        \and  Yutian Chen \\
         DeepMind \\ 
           yutianc@google.com
}

\maketitle

\begin{abstract}
We introduce a framework for online changepoint detection and simultaneous model learning which is applicable to highly parametrized models, such as deep neural networks. 
It is based on detecting changepoints across time by sequentially performing generalized likelihood ratio tests that require only evaluations of simple prediction score functions. This procedure makes use of checkpoints, consisting  of early versions  of the actual model parameters, that allow to detect distributional changes by performing predictions on future data.  We define an algorithm that bounds the Type I error in the sequential testing procedure. We demonstrate the efficiency of our method in challenging continual learning applications with unknown task changepoints, and show improved performance compared to online Bayesian changepoint detection. 
\end{abstract}

\section{Introduction} 

Online changepoint detection is concerned with the problem of sequential detection of distributional changes in data streams, as soon as such changes occur. It can have numerous 
applications ranging from statistical process control, e.g.\ financial times series and medical conditioning monitoring  \citep{Hawkins2003,Bansal2002,Aminikhanghahi:2017,truong2018selective}, to  
problems in machine learning which can involve training very complex and highly parametrized models from a sequence of learning tasks \citep{ring1994continual,robins1995catastrophic,schmidhuber2013powerplay,kirkpatrick2017overcoming}. In this latter application, referred to as continual learning, it is often desirable to train online from a stream of  observations a complex neural network and simultaneously 
detect changepoints that quantify when a task change occurs.    

However, current algorithms for simultaneous 
online learning and changepoint 
detection are not well suited for models 
such as neural networks that can have 
millions of adjustable parameters.  
\comm{
In recent years continual learning 
can involve training a deep neural network from a sequence of different tasks, e.g.\
supervised regression or classification problems
, most current continual learning 
methods typically assume known task changepoints, i.e.\ any time the 
 data distribution changes is given as a ground-truth to the neural network. Instead, here we would like to remove this latter unrealistic assumption  
and develop changepoint detection algorithms that are generally suited for detecting changes when using complex models such as deep neural networks.  
}
For instance, while state of the art Bayesian online changepoint detection algorithms have been developed \citep{Fearnhead2006,Fearnhead2007,Adams07bayesianonline,Caron2012,Sinam13},
 such techniques can be computationally too 
expensive 
to use along with neural networks.  
This is because they are based on Bayesian inference procedures that require selecting suitable 
priors for all model parameters and they rely on applying
accurate online Bayesian inference 
which is generally intractable, unless the model has a simple conjugate form. For instance, the popular techniques in  
\citet{Fearnhead2007,Adams07bayesianonline} 
are tested in simple Bayesian conjugate models where exact integration over the parameters
is feasible. Clearly, such Bayesian computations 
are intractable or very expensive for highly non-linear models such as neural networks, which can contain millions of parameters.

In this article, we wish to develop a framework for joint sequential changepoint detection and online  model fitting, that could be easily applied to arbitrary systems and it will be particularly suited for highly parametrized models such as neural networks. 
The key idea we introduce is to 
sequentially perform 
statistical hypothesis testing by evaluating 
predictive scores under 
cached model checkpoints. Such checkpoints 
are periodically-updated 
copies of the model parameters
and they are used to detect distributional 
changes by performing predictions on future/unseen data (relative to the checkpoint), i.e.\ on data observed  
after a checkpoint and up to present time. An 
illustration of the approach is given 
by Fig.\ \ref{fig:illustration}, while 
detailed description of the method is given in
Section \ref{sec:checkpoints} and Algorithm \ref{alg:checkpoint_detect}.  
In statistical testing for change detection we use generalized 
likelihood ratio tests \citep{csorgo1997limit,jandhyala2002detection} 
where we bound the Type I error (false positive 
detection error) during the sequential testing process. 
The overall algorithm is easy to use and it requires 
by the user to specify two main hypeparameters: the desired bound on the Type I error and the testing window size between a checkpoint and the present time.   

%

\begin{figure*}[!htb]
	\centering
	\begin{tabular}{ccc}
        {\includegraphics[width=0.62\linewidth]  
			{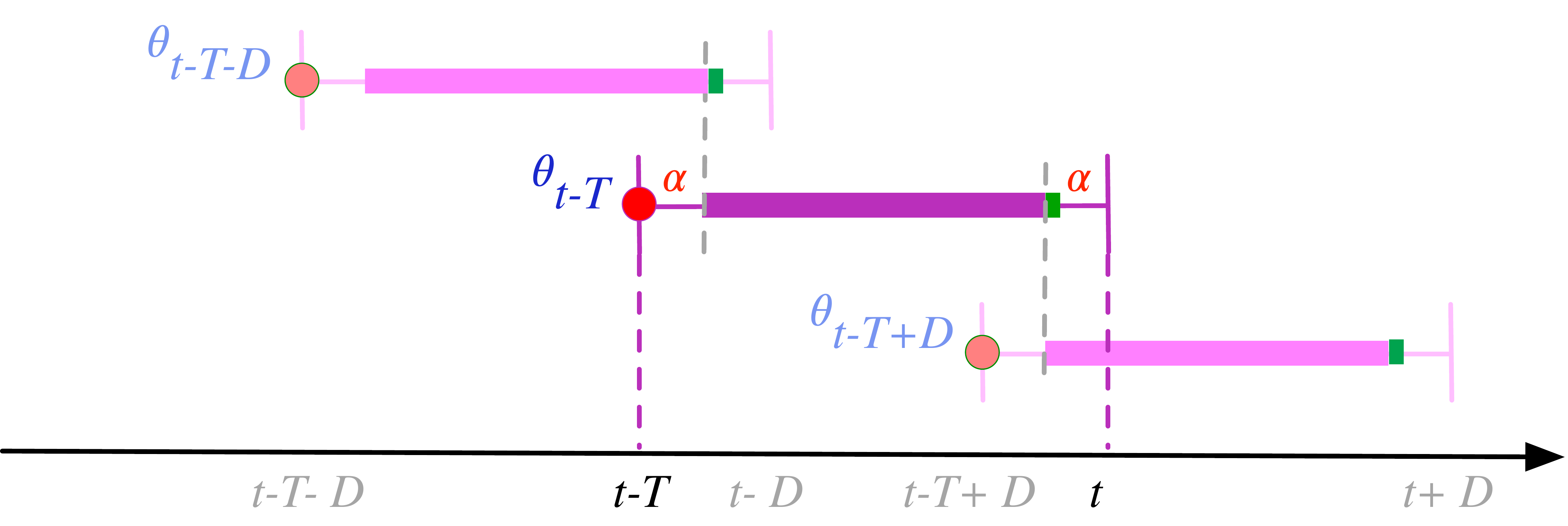}} &
		{\includegraphics[width=0.31\linewidth]  
		 	{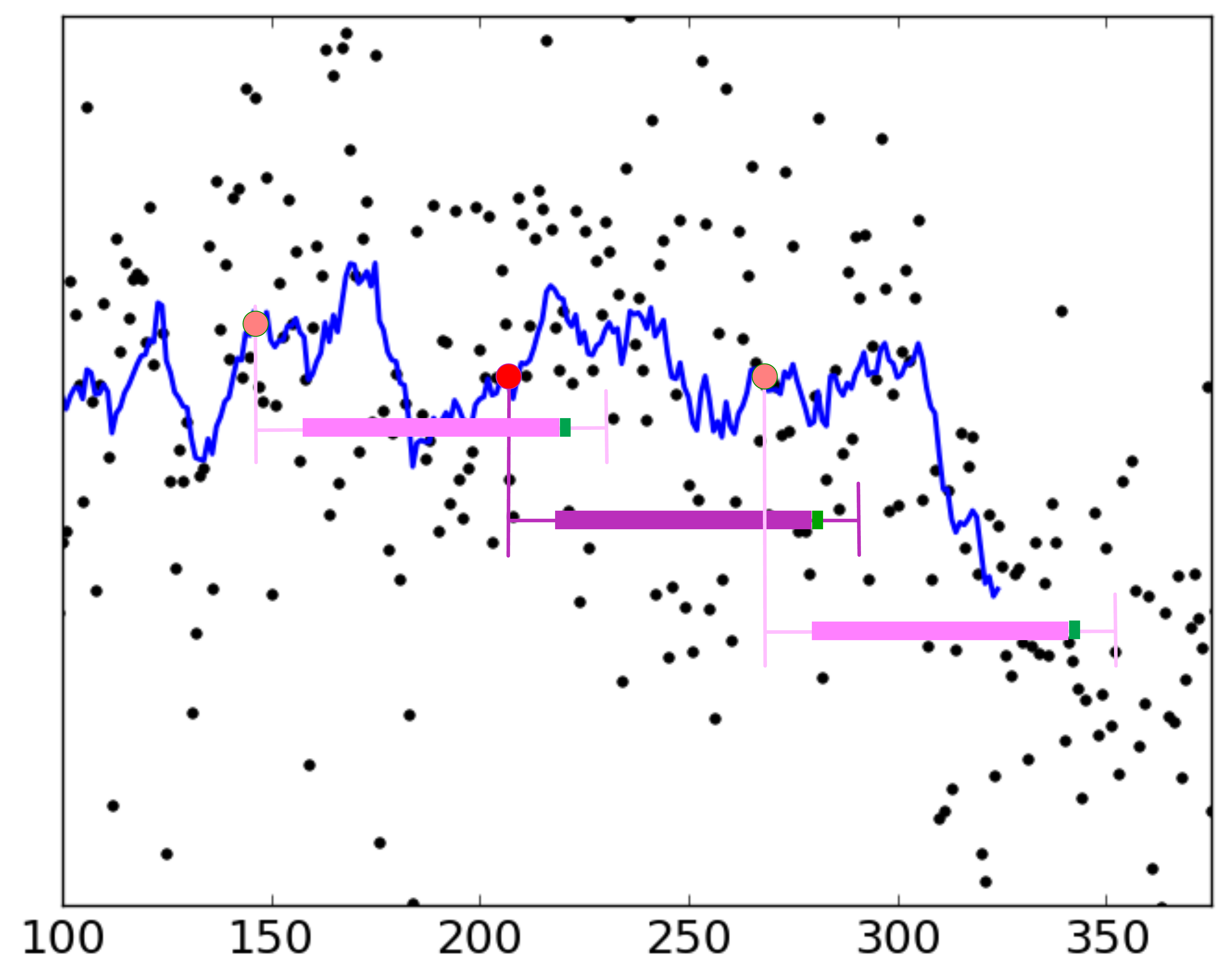}} \\ 
		 	(a) & (b)
	\end{tabular}
	\caption{
	({\bf a}) Visualization of Algorithm 1. 
	Model checkpoints are cached parameters stored (and deleted)
	periodically every $D = T - 2 \alpha$ time steps: 
	$\theta_{t-T-D}, \theta_{t-T}, \theta_{t-T+D}$. 
	Every $D$ steps (e.g.\ $t-D$, $t$, $t+D$) the checkpoint that lags $T$ iterations
	performs the statistical test. E.g.\ given 
	current time $t$ the active checkpoint is $\theta_{t-T}$
	(see highlighted magenta segment)
	which tests for a changepoint in the region $(t-T+\alpha, t-\alpha]$ where $\alpha$ is the minimum sample size in 
	each segment; see Section \ref{sec:offlinedetection}. 
	({\bf b}) An example in a time series dataset (black dots) where the model is a moving average parameter shown by the blue line and the checkpoints are cached values of the moving average indicated by the coloured circles.  	
	} 
	\label{fig:illustration}
\end{figure*}

We demonstrate the efficiency of our method on time 
series as well as on continual learning of neural networks from a sequence of supervised tasks with unknown task changepoints. For these challenging continual learning
problems we also consider a strong baseline for comparison, by using a variant of the Bayesian online changepoint detection algorithm by \citet{Adams07bayesianonline} 
that is easily applicable to complex models. This is done by applying the Bayesian algorithm
on data that correspond to predictive scores provided by the neural network during online training.  
Our proposed method consistently 
outperforms this and others baselines (see Section \ref{sec:experiments}), which 
shows that model checkpoints provide an easy to use and simultaneously  
effective technique for changepoint detection in online learning of complex models. 

The paper is organized as follows. Section \ref{sec:setup} introduces the problem of changepoint 
detection and online model learning. Section \ref{sec:checkpoints} develops our framework  
for changepoint detection using checkpoints 
and Section \ref{sec:continual} considers 
applications to continual learning with neural 
networks. Section \ref{sec:related} discusses 
related work, while Section \ref{sec:experiments} 
provides numerical experimental results and comparisons with other methods. The Appendix contains further details about the method and additional  
experimental results.

\section{Problem Setup
\label{sec:setup}}

\subsection{Streaming Data with Unknown Changepoints}

We consider the online learning problem with unknown changepoints in
a stream of observations $\{y_t\}_{t \geq 0}$. Each $y_t$   
includes an input vector and possibly additional outputs 
such as a class label or a real-valued response. 
For instance, in a supervised learning setting each observation takes the form  $y_t \equiv (x_t,c_t)$ where $x_t$ is the input vector and $c_t$ 
the desired response such as a class label, while in unsupervised learning $y_t$ is an input vector alone, i.e.\ $y_t \equiv x_t$. 
In addition, for many applications, e.g.\ in deep learning \citep{lecun2015deeplearning}, $y_t$ can be a small set or  mini-batch of individual i.i.d.\ observations, i.e.\ $y_t = \{y_t^i\}_{i=1}^b$,
that are received simultaneously at time $t$.

In the generation process of $\{y_t\}_{t \geq 0}$ we assume that there exist certain times, referred to as \emph{changepoints} and denoted by $\{\tau_k \}_{k=1,2,\ldots}$, that result in abrupt changes in data distribution so that $y_{t \in [0, \tau_1)} \overset{iid}{\sim} 
\mathcal{P}_1, y_{t \in [\tau_{1}, \tau_2)} \overset{iid}{\sim}
\mathcal{P}_2$ and in general 
\begin{equation}
y_{t \in [\tau_{k-1}, \tau_k)} \overset{iid}{\sim} 
\mathcal{P}_k, \quad  k=1,2,\ldots
\label{eq:datasegment}
\end{equation}
where $\mathcal{P}_{k-1} \neq \mathcal{P}_k$ and with the convention $\tau_0=0$. Each $\mathcal{P}_k$ is the  segment
or task-specific distribution that generates the 
$k$-th data segment. These 
assumptions are often referred to as partial exchangeability or the product partition model \citep{Barry1992}.
To learn from such data we wish to devise  schemes that can adapt online a parametrized model  
without knowing the changepoints $\tau_k$
and the distributions $\mathcal{P}_k$.   
Accurate sequential detection of changepoints can be useful since, knowing them, the learning system can dynamically decide to switch to a different parametric model or add new parameters to a shared model and etc.  
In Section \ref{sec:checkpoints}, 
we introduce a general online learning
and changepoint detection algorithm  
suitable for arbitrary models ranging from simple single-parameter models 
to complex deep networks having millions of parameters.

\subsection{Online Model Learning with Changepoints}

We consider a probabilistic 
model $p(y|\theta)$ with parameters $\theta$ that we wish 
to train online 
and simultaneously use it to detect 
the next changepoint $\tau_k$.
Online training of $\theta$ means that for each observation $y_t$ 
we perform, for instance, a gradient update
\begin{equation}
\theta_t \leftarrow \theta_{t-1} - \rho_{t} \nabla \ell(y_t; \theta_{t-1}),
\label{eq:updaterule}
\end{equation}
where $\rho_t$ is the step size or learning rate. 
Given the non-stationarity of the learning problem this sequence should not satisfy the 
Robbins-Monro conditions \citep{robbins1951} and, for instance, $\rho_t$ could be constant through time.
The loss function in Eq.\ \eqref{eq:updaterule} is typically the negative log-likelihood function, i.e.\ 
$$
\ell(y_t;\theta_{t-1}) = - \log p(y_t | \theta_{t-1})
$$
and $\theta_t$ denotes the parameter values 
after having seen $t$ observations including the most recent $y_t$.
We will refer to 
the evaluations of the loss $\ell(y_{t'}, \theta_{t-T})$, or any other score function $v(y_{t'}, \theta_{t-T})$ on 
any future data $y_{t'}$ (relative to $\theta_{t-T}$), with $t' > t - T$, as {\em prediction} scores. 
Notice that given $y_{t'}$s have been drawn i.i.d.\ from the same data segment, the prediction scores are also i.i.d. random variables.

Suppose at time $t$ the data segment or \emph{task} is $k$ and we observe $y_t$, 
$t \geq \tau_{k-1}$ where $\tau_{k-1}$ is the most recently detected 
changepoint, i.e.\ 
when the $k$-th task started 
as shown in Eq. \eqref{eq:datasegment}.
If we decide that data $y_t$ comes from a new task $k+1$, we could set $\tau_k = t$, instantiate 
a new model with a fresh set of parameters $\theta^{(k+1)}$ and repeat the process. 
All these models could have completely separate parameters, i.e.\ 
$\theta^{(k)} \cap \theta^{(k')} = \emptyset, \forall k \neq k'$ or allow parameter sharing,
i.e.\ $\theta^{(k)} \cap \theta^{(k')} \neq \emptyset$, 
as further discussed in  
Section \ref{sec:continual} 
where we describe applications to continual learning.

\section{Changepoint Detection with Checkpoints}
\label{sec:checkpoints}

Throughout this Section 
we will be interested to detect 
the next changepoint $\tau_k$. Thus, to simplify notation we will drop index
$k$ and write this changepoint as $\tau$ and the current parameters as $\theta$ when it does not cause confusion. We will also assume that the previously detected changepoint is at time zero. 

The iterative procedure for changepoint detection with checkpoints is illustrated in Fig.\ \ref{fig:illustration}.
This algorithm assumes that 
together with the current parameter values $\theta_t$  
we cache in memory one or multiple copies of early values of the parameters referred to as 
model checkpoints or simply checkpoints. Checkpoints 
are cached and deleted periodically and statistical testing
for changepoint detection
is also performed periodically. When at iteration $t$, 
where model parameters are $\theta_t$, 
we need to perform a changepoint detection test
the algorithm 
activates the checkpoint $\theta_{t-T}$, that
has been cached 
$T$ iterations before. This checkpoint 
forms predictions on all subsequent data (not seen by the checkpoint) 
in order to detect a change in the data distribution
that possibly occurs in the data segment in 
$(t-T, t]$. 
Pseudo-code of the algorithm is provided in Algorithm \ref{alg:checkpoint_detect}. 

In the remaining of this Section we will be discussing in detail how Algorithm \ref{alg:checkpoint_detect} works. Some useful 
summarizing remarks to keep in mind are the following. The algorithm caches checkpoints every $D = T - 2 \alpha$ iterations with the first checkpoint cached at $t=0$. $T$ is the window size, $D$ is the stride and $\alpha>0$ is the minimum sample size when computing the testing statistics. The first testing occurs at time $t=T$, i.e.\ when the data buffer $\mathcal{B}_t$ becomes full 
and the first cached checkpoint used is the initial parameter values $\theta_0$. This constrains also the minimum size of the data segment (i.e.\ the distance between two consecutive changepoints) to be $T$. After the first test, testing occurs every $D$ iterations and given that each checkpoint is deleted after a test the number of checkpoints in memory is roughly $T/D$.

\subsection{Offline Changepoint Detection in a Window}

Suppose a sliding window of observations $y_{t'}, t - T < t' \leq t$ (recall that $y_{t'}$ can generally be a set or mini-batch of $b$ i.i.d.\ 
individual observations) of size $T$, i.e.\ all data observed strictly after the model checkpoint $\theta_{t-T}$.  
Given a scalar prediction score function: 
$$
v(y_{t'},\theta_{t-T}) = \frac{1}{b}\sum_{i=1}^b v(y_{t'}^i,\theta_{t-T}),
$$
or $v_{t'}$ for short, we consider the offline changepoint detection problem in the interval $(t-T, t]$ with the independent and normal distribution assumption:
\begin{assumption}
\label{ass:normality}
$v_{t'} \sim \mathcal{N}(\mu_{t'}, \sigma_{t'}^2) \textrm{ independently for all $t'$}.$
\end{assumption}

We consider the following hypothesis testing problem with unknown change time, mean and variance:
\begin{itemize}
    \item $\mathcal{H}_0: \exists \mu, \sigma^2 \textrm{ s.t. } \mu_{t'} = \mu, \sigma_{t'}^2 = \sigma^2, \forall t' \in (t-T, t]$.
    \item $\mathcal{H}_1: \exists \tau \in (t-T+\alpha, t-\alpha], \mu_1, \sigma_1^2, \mu_2, \sigma_2^2$ s.t. \\
    $\mu_{t'} = \mu_1, \sigma_{t'}^2 = \sigma_1^2,~~ \forall t' \in (t-T, \tau)$, \\
    $\mu_{t'} = \mu_2, \sigma_{t'}^2 = \sigma_2^2,~~ \forall t' \in [\tau,  t]$,
\end{itemize}
where $\alpha \in (0, T/2)$ is the minimum sample size in each segment of $\mathcal{H}_1$ for estimating the mean and variance. Using a model checkpoint and applying the testing on predictions is important to satisfy the independent assumption on scores.

Following the generalized likelihood ratio (GLR) test, we denote by $\Lambda_\tau$ the likelihood ratio of the two hypotheses at a changepoint of $\tau$ with the unknown variables taking the maximum likelihood estimates,
\begin{equation}
\Lambda_\tau = \frac{\underset{\mu, \sigma^2}{\sup}  p(\bfv_{(t-T, t]} | \mu, \sigma^2) }{\underset{\mu_1, \sigma_1^2}{\sup} p(\bfv_{(t-T, \tau)} | \mu_1, \sigma_1^2) \underset{\mu_2, \sigma_2^2}{\sup} p(\bfv_{[\tau, t]} | \mu_2, \sigma_2^2) },
\label{eq:lambda_tau}
\end{equation}
and compute the statistics as follows:
\begin{align}
Z =& \underset{\tau \in (t-T+\alpha,t-\alpha]}{\max} (-2 \log \Lambda_\tau) \nonumber\\
=& \underset{\tau \in (t-T+\alpha,t-\alpha]}{\max} \big\{ T \log S(\bfv_{(t-T, t]}) \nonumber\\
& - (\tau - t + T - 1) \log S(\bfv_{(t-T, \tau)}) \nonumber\\
& - (t - \tau + 1) \log S(\bfv_{[\tau,t]}) \big\},
\label{eq:Z}
\end{align}
where $S$ is the sample variance, $\bfv_{(t-T, \tau)}$ denotes the set of all values $v_{t'}, t-T < t'  < \tau$, $\bfv_{[\tau,t]}$ 
the values $v_{t'}, \tau \leq t' \leq t$ 
and $\bfv_{(t-T, t]}$ their union. 

Asymptotic distribution of the statistics $Z$ as $T\rightarrow \infty$ has been well studied in the literature for the normal distribution of $v_{t'}$ \citep{csorgo1997limit,jandhyala2002detection}. For a finite window size $T$, we can also compute the critical region, $Z > h(\delta)$ at a given confidence level $1-\delta$ numerically; see the Appendix. When the null hypothesis is rejected, we claim there is a changepoint in the current window $(t-T, t]$, and the changepoint is selected with $\tau = \arg\max_{\tau'} (-2 \log \Lambda_{\tau'})$.

It is important to note that the alternative hypothesis $\mathcal{H}_1$ is not a complement of the null, and we consider the candidate changepoint $\tau$ in a subset $(t-T + \alpha, t-\alpha] \subset (t-T, t]$ for reliable estimate of sample mean and variance. This means that when a true changepoint exists in the right border $[t-\alpha+1,t]$, it might cause a rejection and show up in the nearest location on the subset, i.e.\ $t-\alpha$, which subsequently could increase the error of the changepoint location estimation. To reduce this effect we can compute $\Lambda_{\tau}$ in the extended subset $(t-T + \alpha, t-\alpha+1]$, and do not reject $\mathcal{H}_0$ if $Z < -2\log \Lambda_{t-\alpha+1}$ ($Z$ is still taken in $(t-T+\alpha, t-\alpha]$). Notice that there are $\alpha$ samples in the right side when $t' = t-\alpha+1$, satisfying our requirement for the minimum sample size.

We repeat the offline detection using a sliding window of size $T$ with a stride $D=T-2\alpha$. This ensures that every time location will be included in the candidate subset for exactly one test. An illustration of this is shown in
Fig.\ \ref{fig:illustration} where the green border is precisely the location $t-\alpha+1$ in the extended subset, 
which is ignored if the maximum occurs there, but it could be accepted 
in the next iteration where the green location becomes 
the first location in the new subset. Similarly, the possibility that a changepoint occurs in the left border $(t-T, t-T+\alpha]$ can be detected in a previous window.

\begin{algorithm}[tb]
\caption{Changepoint detection with checkpoints}
\label{alg:checkpoint_detect}
\begin{small}
\begin{algorithmic}
    \STATE {\bfseries Procedure:} $\mathrm{changepoint\_detection}(\theta_0, \alpha, T, \delta, \eta, \mathrm{update\_step})$
    \STATE {\bfseries Input:} Initial parameters $\theta_0$, minimal sample size $\alpha$, window size $T$, error schedule parameters $\{\delta, \eta\}$, optimization step function $\mathrm{update\_step}$  
    \STATE {\bfseries Output:} Changepoint location $\tau^*$, model parameter $\theta$
    \STATE Initialize: test region size $D = T - 2\alpha$, test index $i = 0$, time step $t=0$, data buffer $\mathcal{B}_0 = \emptyset$.
    \FOR{$t=1,2,\ldots$}
        \STATE Receive mini-batch $y_t$, update data buffer $\mathcal{B}_t = \{y_{t'}: \max(0, t-T) < t' \leq t\}$
        \STATE $\theta_t = \mathrm{update\_step}(\theta_{t-1}, y_t)$  
        \STATE {\bfseries if} $t ~\mathrm{mod}~ D = 0$ {\bfseries then} cache checkpoint $\theta_t$
        \IF{$t = i D + T$}
            \STATE Compute scores with cached checkpoint $\theta_{t-T}$: 
            \STATE ~~~~$v_{t'} \equiv v(y_{t'}, \theta_{t-T})$, $t-T< t' \leq t$ 
            \STATE $\delta_i = (1-\eta) \eta^i \delta$
            \STATE $(\mathrm{reject},\tau^*) = \mathrm{offline\_detection}(\bfv_{(t-T,t]}, \alpha, \delta_i)$
            \IF{$\mathrm{reject}$}
                \STATE {\bfseries return} $(\tau^*, \theta_t)$
            \ENDIF
            \STATE $i \leftarrow i + 1$, delete checkpoint $\theta_{t-T}$
        \ENDIF
   \ENDFOR
\end{algorithmic}
\end{small}
\end{algorithm}

\begin{algorithm}[tb]
\caption{Offline changepoint detection}
\label{alg:checkpoint_detect_subroutine}
\begin{small}
\begin{algorithmic}
    \STATE {\bfseries Subroutine:} $\mathrm{offline\_detection}(\bfv, \alpha, \delta)$
    \STATE {\bfseries Input:} Scores $\bfv_{(t-T,t]}$, minimal sample size $\alpha$, error $\delta$
    \STATE {\bfseries Output:} Boolean $\mathrm{reject}$, candidate location $\tau^*$
    \STATE Compute the threshold $h=\mathrm{quantile}(1-\delta)$
    \FOR{$t'=t-T+\alpha+1, \dots, t-\alpha+1$}
        \STATE Compute $\Lambda_{t'}$ using Eq.\ \eqref{eq:lambda_tau}
    \ENDFOR
    \STATE Compute $Z$ using \eqref{eq:Z}, $\tau^*=\underset{t-T+\alpha <\tau \leq t-\alpha}{\mathrm{argmax}}(-2 \log \Lambda_\tau)$
    \STATE $\mathrm{reject} = Z>h \text{ and } Z > -2\log \Lambda_{t-\alpha+1}$
    \STATE {\bfseries return} $\mathrm{reject}, \tau^*$
\end{algorithmic}
\end{small}
\end{algorithm}

\subsection{Online Changepoint Detection across Windows
\label{sec:offlinedetection}}

As the interval between two changepoints spans over multiple test windows, we would like to control the overall error $\delta$ of making a false rejection of the null hypothesis, that is making a false claiming that data distribution changes, in every data segment. Since the model checkpoints change at every test window, it is difficult to apply a standard sequential likelihood ratio test across windows. Instead, we select the confidence level with an annealing schedule so that the overall error is bounded:
\begin{equation}
    \delta_i = (1-\eta) \eta^i \delta, \label{eq:delta_schedule}
\end{equation}
where $\delta_i$ is the confidence level for the $i$-th (0-based) test window after a new task is detected and $\eta$ is the decaying rate.

\begin{proposition}
\label{prop:error}
Given Assumption~\ref{ass:normality} holds, the probability of making a Type I error (false changepoint detection) by Algorithm  \ref{alg:checkpoint_detect} between two real changepoints is upper bounded by $\delta$.
\end{proposition}
\begin{proof}
Let $y_1, y_2, \dots, y_N$ be the segment of data stream with the same distribution $\mathcal{P}$ where $N$ is the time of last mini-batch of data, $y_N = \{y_N^i\}_{i=1}^b$. Offline change point detection is conducted in windows $(0, T]$, $(D, T+D]$, $(2D, T+2D]$, $\dots$, $(nD, T+nD]$ before a changepoint occurs, where $n$ is the maximum integer with $T+nD \leq N$. Under Assumption \ref{ass:normality}, the probability of rejecting the null hypothesis at the $i$-th testing window with input error argument $\delta_i$ is
\begin{equation}
    Pr(\mathrm{reject}_i)  \leq Pr(Z > h(\delta_i)) = \delta_i,
\end{equation}
where the first inequality is due to the possibility to ignore the rejection when $Z \leq -2\log \Lambda_{t - \alpha + 1}$, and the second equality follows the definition of $h$ in Algorithm \ref{alg:checkpoint_detect_subroutine} as the $1-\delta$ quantile.

The probability that the null hypothesis is rejected in at least one testing window is then upper bounded with the union bound by
\begin{align}
    Pr(\cup_{i=0}^n \{\mathrm{reject}_i\}) 
    & \leq \sum_{i=0}^n Pr(\mathrm{reject}_i) = \sum_{i=0}^n \delta_i \nonumber\\
    & = \sum_{i=0}^n (1-\eta) \eta^i \delta < \delta.
\end{align}
\end{proof}
The time complexity of running Algorithm \ref{alg:checkpoint_detect} on a datastream of length $t$ is $O(tbT/D)$ where $b$ is the mini-batch size and $D=T-2\alpha$ is the stride of the sliding window, and the space complexity for storing the checkpoint and data buffer is $O(T/D S_m + Tb S_y)$ where $S_m$ and $S_y$ denotes the size of a model checkpoint and a data point.

 
\subsection{Setting the Hyperparameters and Prediction Scores}
\label{sec:hyper-parametres}

There are a few hyperparameters in our proposed algorithm, including the window size $T$, minimum sample size in a window $\alpha$, Type I error $\delta$, and the error decaying factor $\eta$. 

A large window size provides more data for every offline detection and usually leads to a higher accuracy. However, the space complexity increases linearly as the window size in order to keep a data buffer of size $T$, and it has to be upper bounded by our prior assumption on the minimum distance between two consecutive changepoints. Also, when the score function $v(y_{t'},\theta_{t-T})$ requires a good model fit in 
order to be discriminative between tasks, 
a smaller window size can be beneficial at the beginning of a new task because it would update the model checkpoint more frequently and thereby improve the discriminative power in detecting a changepoint more quickly. We study the effect of $T$ empirically in our experiments.

A sufficiently large minimum sample size $\alpha$ is important to obtain reliable estimate of the sample variance and stabilize the distribution of the statistics $Z$. But too large value in $\alpha$ reduces the range of candidate locations and decreases the power in a single offline detection test. Also, because the sliding window has a stride of $D = T - 2\alpha$, the time complexity increases with $\alpha$. In our experiments, we use a default value $\alpha = \lfloor T/4 \rfloor$, giving $D \approx T/2$.
Notice that with such default settings only two checkpoints are needed to be kept in memory (since $T/D = 2$) resulting in 
small memory cost. 

Given a total Type I error $\delta$, the decaying factor $\eta$ controls the exponential distribution of the error across windows with mean $D / \log(1/\eta)$. In principle, we would like the mass of the error to be distributed in the support of our prior about the changepoint frequency. In lack of this knowledge, we use $\eta=0.99$ in all the experiments.

The prediction score function $v(y_{t'}, \theta_{t-T})$ must be discriminative with respect to data streams from different tasks. When the model parameter $\theta_t$ is well fitted to the current task, a properly chosen score function is usually sensitive to the change of the task. Nevertheless, we emphasize that being fitted to the current task is not a necessary condition for our changepoint detection method to work. As demonstrated in the example of Section \ref{sec:illustrative_example}, our algorithm in some cases can detect the changepoints robustly regardless of the learning rate
in the update rule in Eq.\ \eqref{eq:updaterule} that affects how well the model is fitted to the data.

A key assumption about our detection algorithm is the normal distribution of the score function defined on every mini-batch of data.
In experiments with continuous observations we use the average negative log-likelihood as score $v_{t'}= 1/b \sum_{i=1}^b -\log p(y_{t'}^i|\theta_{t-T}))$, and in experiments with discrete observations, we find it works better by applying another logarithm operation as 
$$
v_{t'}= 1/b \sum_{i=1}^b \log (- \log Pr(y_{t'}^i|\theta_{t-T}) + \epsilon),
$$
 where $\epsilon>0$ is a small jittering term for numerical stability. 
Fig.\ \ref{fig:score_histogram} shows typical histograms of scores in a testing window from our
continual learning experiments with real-world data; see Section \ref{sec:experiments}. We also apply the D'Agostino and Pearson's normality test \citep{d1971omnibus,d1973tests} on a sample of 100 scores in this setting and show the p-value in the caption of each plot. It is clear that the normality improves with a larger size $b$ of mini-batch due to the central limit theorem, and the distribution of scores in the log-domain is closer to a normal distribution. We show in the experiments that the performance of our detection improves significantly with the mini-batch size.

\begin{figure*}[tbh]
\centering
\subfigure[Batch size = 10, p-value = 0.002.]{
\includegraphics[width=.26\textwidth]{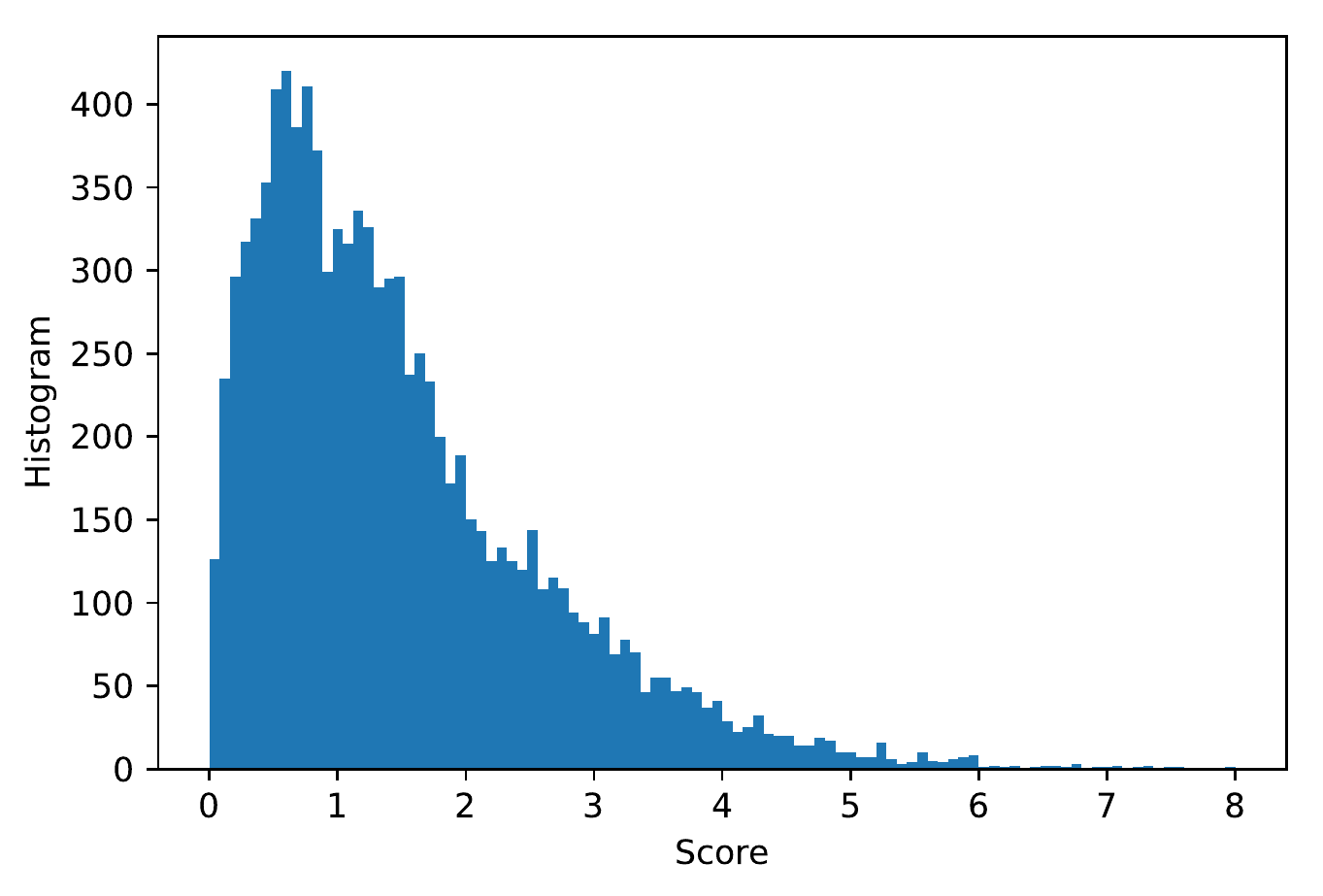}
\label{fig:hist_score_bs10}
}
~
\subfigure[Batch size = 50, p-value = 0.001.]{
\includegraphics[width=.26\textwidth]{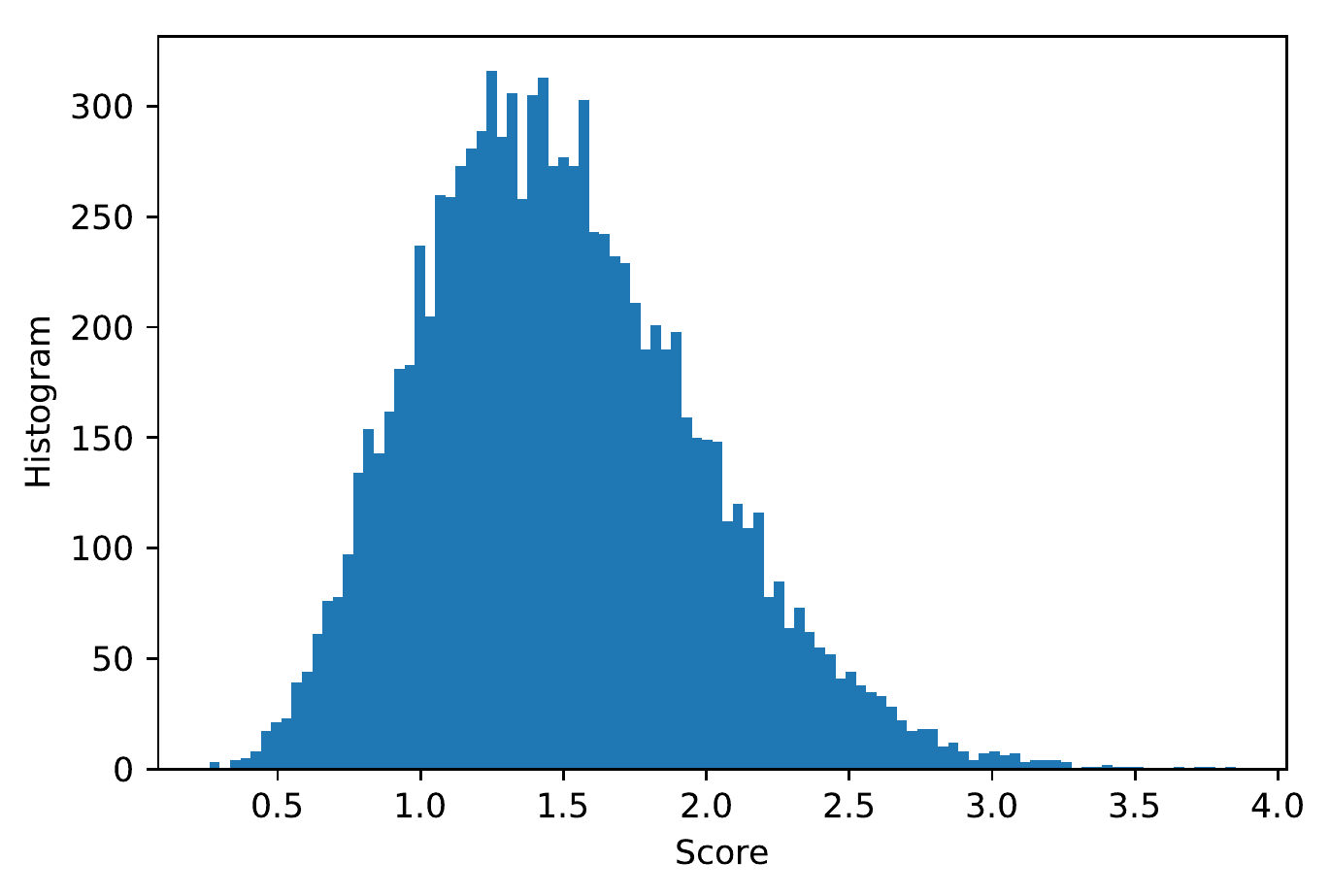}
\label{fig:hist_score_bs50}
}
~
\subfigure[Batch size = 100, p-value = 0.32.]{
\includegraphics[width=.26\textwidth]{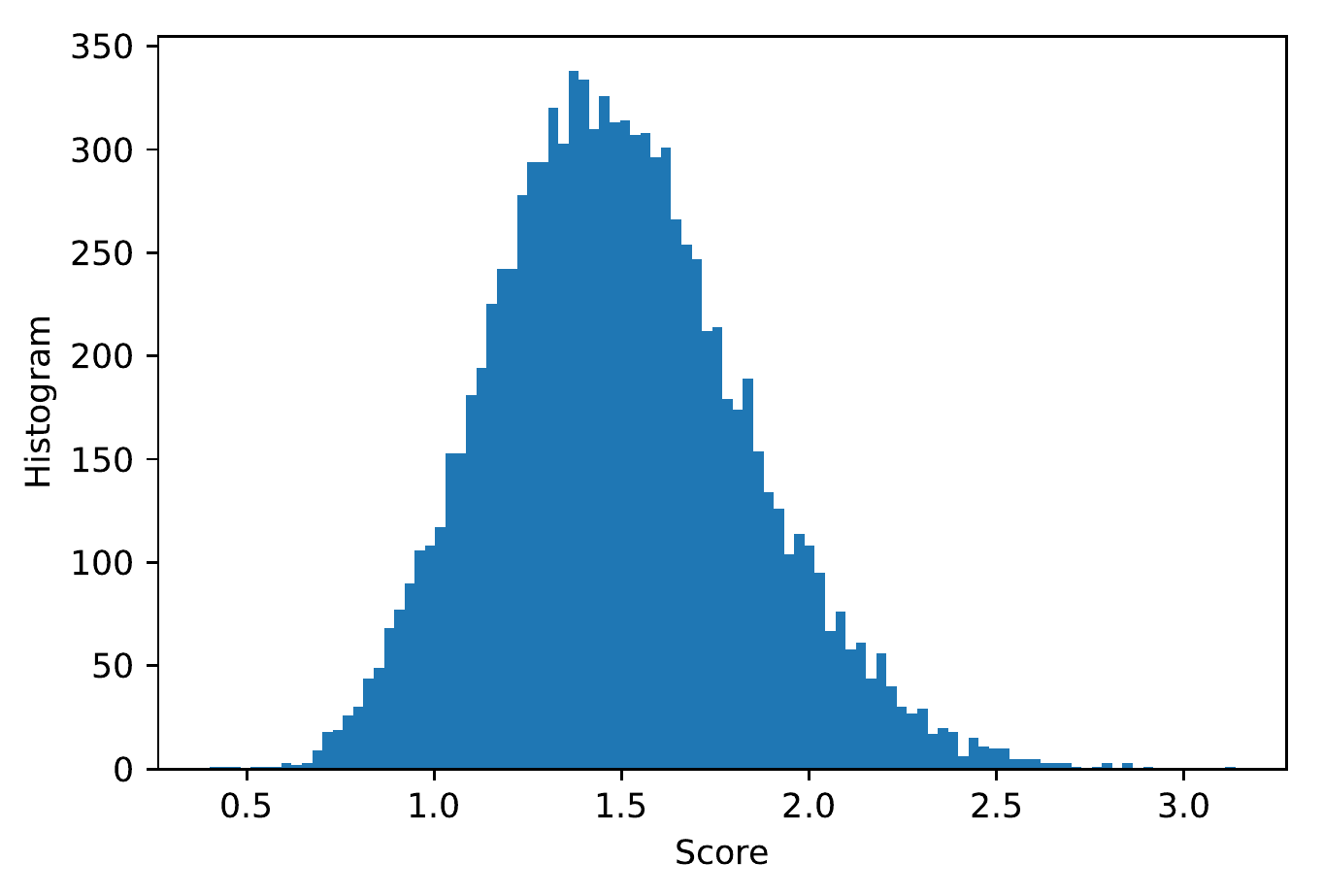}
\label{fig:hist_score_bs100}
}
\\
\subfigure[Batch size = 10, p-value = 0.42.]{
\includegraphics[width=.26\textwidth]{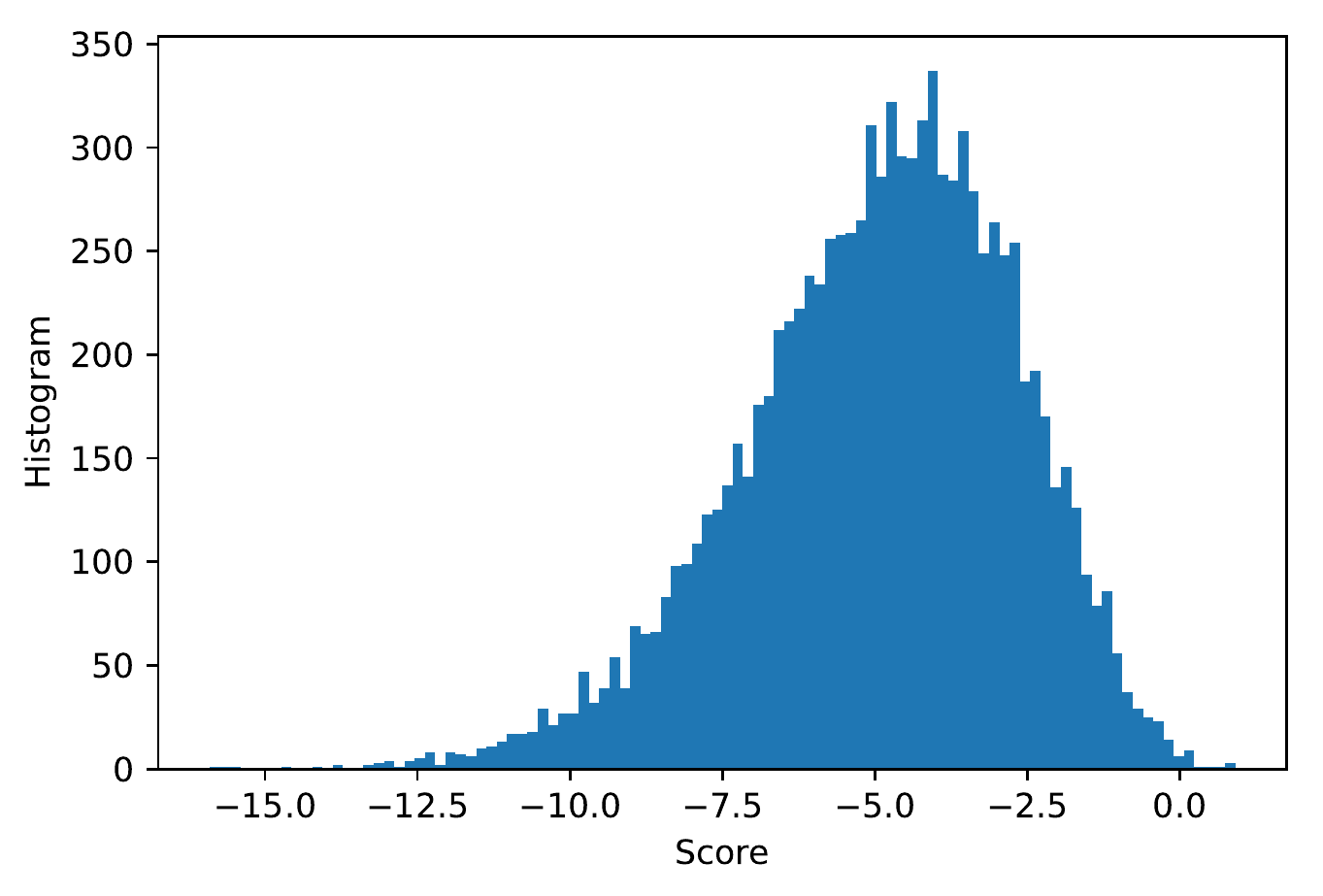}
\label{fig:hist_score_bs10_042}
}
~
\subfigure[Batch size = 50, p-value = 0.93.]{
\includegraphics[width=.26\textwidth]{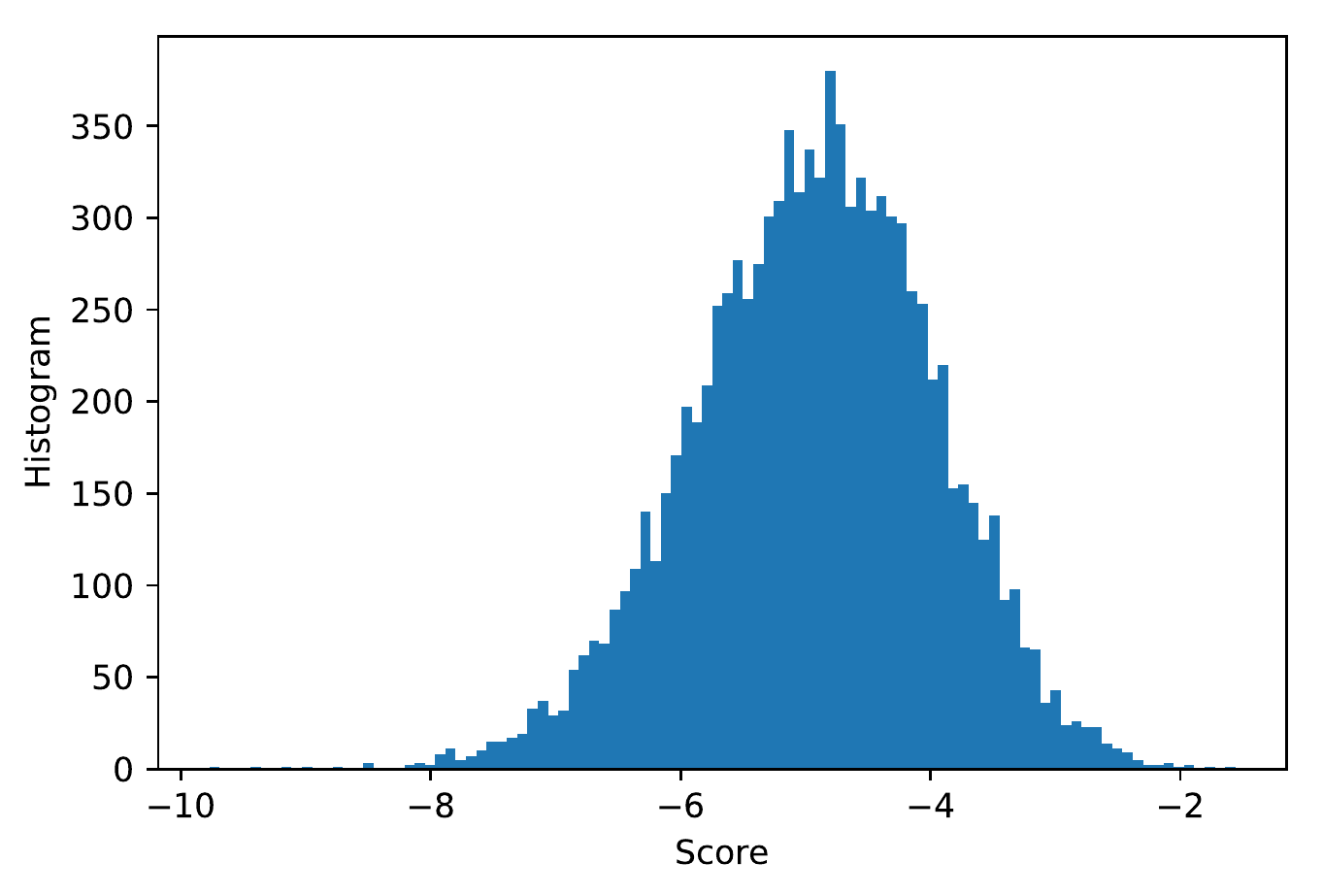}
\label{fig:hist_score_bs50_093}
}
~
\subfigure[Batch size = 100, p-value = 0.29.]{
\includegraphics[width=.26\textwidth]{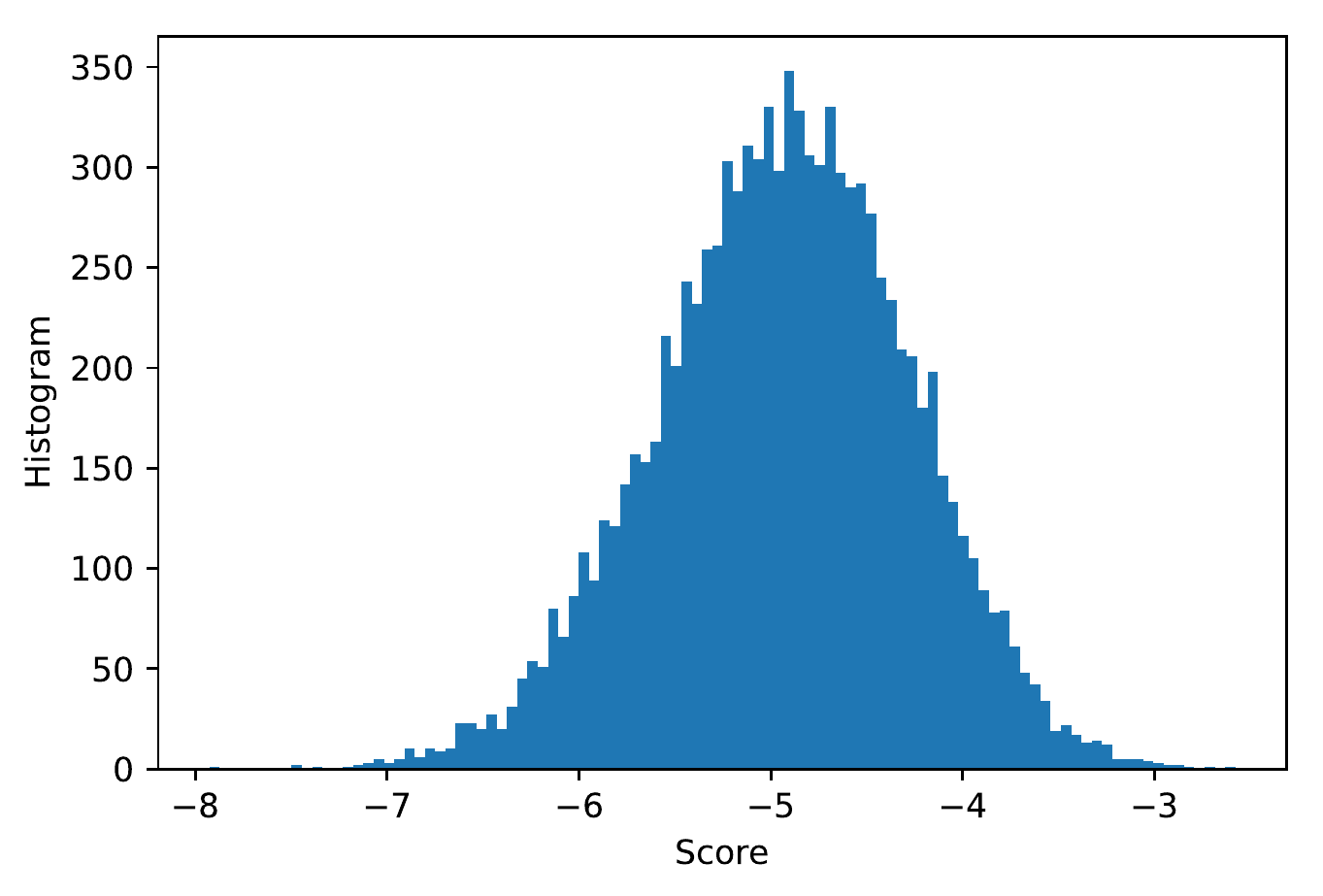}
\label{fig:hist_score_bs100_029}
}
\caption{Typical histogram of scores in the Split-MNIST experiment when no changepoint occurs. Top row use the score of the mean negative log-likelihood, and bottom row applies another logarithm transformation. The mini-batch size from left to right is 10, 50 and 100 respectively. p-value of the normality test based on 100 samples is shown in the subfigure caption.}
\label{fig:score_histogram}
\end{figure*}

\section{Application to Continual Learning \label{sec:continual}}

To test our method on a challenging 
online model fitting and changepoint detection problem
we consider continual learning (CL) 
 \citep{ring1994continual,robins1995catastrophic,schmidhuber2013powerplay,goodfellow2013empirical}, which 
 requires training neural networks on a sequence of tasks. Many recent CL methods, see e.g.\ \citep{kirkpatrick2017overcoming,nguyen2017variational,rusu2016progressive,li2017learning,farquhar2018towards},
typically assume known task changepoints, also called  task boundaries. 
Instead, here we wish to train a CL method without knowing the task boundaries 
and investigate whether we can accurately 
detect the changepoint locations that quantify when the 
data distribution is changing from one 
task to the next.     
%
%
The sequential learning and changepoint 
detection Algorithm \ref{alg:checkpoint_detect}
can be easily incorporated to existent 
CL algorithms, since the essential component 
of the algorithm is the prediction score function $v(\cdot)$
used in hypothesis testing.
In the following, we combine our algorithm with a standard experience replay  CL method \citep{robins1995catastrophic,robins1998catastrophic,lopez2017gradient,rebuffi2017icarl}, 
which regularizes stochastic gradient descent model training  for the current task by replaying a small subset of observations from previous tasks, as detailed next.  

As the main structure for the CL model 
we consider a feature vector $\phi(x;\theta^s)
\in \mathbbm{R}^M$,
obtained by a neural network with parameters $\theta^{s}$, where these parameters are shared across all tasks. 
For each detected $k$-th task there is
a set of task-specific or private parameters $\theta^{p,k}$, which are added dynamically into the model each time 
our algorithm returns a detected changepoint indicating the beginning of a new task. In the experiments,  
we consider CL problems where each task is a 
binary or multi-class classification problem, so that 
each $\theta^{p,k}$ is a different head, i.e.\ a set of final output parameters, 
attached to the main network consisted of the feature vector $\phi(x;\theta^s)$. 
For instance, if the task is multi-class classification, then $\theta^{p,k}$ is a matrix of size $C \times M$, where $C$ denotes the number of classes. In this case, such 
 task-specific parameter  allows the computation of the softmax or multinomial
logistic regression likelidood 
$$
p(c|x,k) = \frac{\exp\{ \sum_{m=1}^M [\theta^{p,k}]_{cm} \phi_m(x;\theta^s) \}}
{\sum_{c'=1}^C \exp\{ \sum_{m=1}^M [\theta^{p,k}]_{c'm} \phi_m(x;\theta^s) \}}, 
$$
that models the categorical probability distribution for classifying input data points from the $k$-th task.  

We assume that the CL model is continuously trained so that tasks occur sequentially and they are separated by random changepoints. For simplicity, we also assume that previously seen tasks never re-occur. 
At time $t=0$, the shared parameters $\theta^s$ are initialized to some random value, and the parameters $\theta^{p,1}$ of the first task are also initialized arbitrarily, e.g.\ to zero or randomly. Then, learning of the model parameters progresses so that whenever a changepoint $\tau_k$ occurs a fresh set of 
task-specific parameters $\theta^{p,k}$ is instantiated while all existing parameters, such the shared parameters $\theta^s$, 
maintain their current values, i.e.\ they are not re-initialized. However, this continuous updating can cause the shared feature vector $\phi(x;\theta^s)$ to yield poor predictions on early tasks, a phenomenon known in the literature of neural networks as catastrophic forgetting  \citep{robins1995catastrophic,goodfellow2013empirical,kirkpatrick2017overcoming}, and the prevention of this is one major challenge CL methods need to deal with.

More specifically, at each time instance the model receives a mini-batch of training 
examples $y_t = \{c_t^i, x_t^i\}_{i=1}^b$, where $c_{t}^i$ 
is a class label and $x_t^i$ is an input vector. At each 
time step the current detected task is $k$, so 
that in the shared feature vector $\phi(x;\theta^s)$ 
we have attached so far $k$ heads each with 
task-specific parameters $\theta^{p,i}$, $i=1,\ldots,k$. The full set of currently instantiated parameters is denoted by $\theta^{(k)} = (\theta^s, \{\theta^{p,i}\}_{i=1}^k)$ to emphasize the dependence on the $k$-th task. Training with the current 
$k$-th task is performed by using the standard negative log-likelihood, 
i.e.\ cross-entropy loss, regularized by adding a sum of \emph{replay
buffers}, which correspond to negative log-likelihoods terms evaluated at small data subsets from all previous tasks, 
 \begin{equation}
  L(\theta^{(k)}) = 
 L_{k}(\{c_t^i, x_t^i\}_{i=1}^b; \theta^{p,k},\theta^s) 
 +  \lambda \sum_{i=1}^{k-1}  L_i (\mathcal{R}_i; \theta^{p,i}, \theta^s), 
 \label{eq:CLloss}
 \end{equation} 
 where $\lambda>0$ is a regularization parameter and 
 each $L_{i}(\cdot)$ is a sum of negative log-likelihood terms over the 
 individual data points. Each $\mathcal{R}_i$ 
 is a small random subset of data from the $i$-th task that is stored as soon as this task  is detected and then used as an  
 experience replay (to avoid forgetting the $i$-th task) 
 when training in future tasks. 
 
 Pseudo-code of the whole procedure for training the CL model with simultaneous changepoint detection based on our checkpoint framework is outlined in Algorithm \ref{alg:continual_learning}. For simplicity in Algorithm   \ref{alg:continual_learning} we assumed that 
 the replay buffers $\mathcal{R}=\{\mathcal{R}_i\}_{i=1}^{k-1}$
 are global variables that affect the subroutine  $\mathrm{cl\_update\_step}$ without having to be passed as inputs. A second simplification is that each task replay buffer $\mathcal{R}_k$ in practice is actually created inside Algorithm \ref{alg:checkpoint_detect}, where a few data mini-batches of the current task are stored into the fixed-size memory to form $\mathcal{R}_k$.\footnote{This second simplification was made to keep the structure of Algorithm \ref{alg:checkpoint_detect} in its general form, while the minor modification regarding the replay buffers is only needed for this specific CL application.}

\begin{algorithm}[tb]
\caption{Continual learning}
\label{alg:continual_learning}
\begin{small}
\begin{algorithmic}
    \STATE {\bfseries Procedure:} $\mathrm{continual\_learning}$
    \STATE {\bfseries Input:} Initial shared model parameter $\theta^s$, parameters for $\mathrm{changepoint\_detection}$: $\alpha$, $T$, $\{\delta, \eta\}$
    \STATE {\bfseries Output:} Model parameters $\Theta$, list of changepoints $\mathrm{T}$
    \STATE Initialize: list of  parameters $\Theta = \theta^s$, list of replay buffers $\mathcal{R}=[]$, and list of changepoints $\mathrm{T} = []$
    \FOR{$k=1,2,\ldots$}
        \STATE Initialize task private parameter $\theta^{p,k}$
        \STATE Concatenate all current parameters: 
        $\theta^{(k)} = (\theta^s, \{\theta^{p,i}\}_{i=1}^k)$ 
        \STATE $\tau^*, \! \theta^{(k)} \! = \!
        \mathrm{changepoint\_detection}(\theta^{(k)}, \alpha, T, \delta, \eta, \! \mathrm{cl\_update\_step})$
        
        \STATE Append $\tau^*$ to list $\mathrm{T}$ 
        
         \STATE Construct a task replay buffer $\mathcal{R}_k$ and append it to the list $\mathcal{R}$ 
         
         \STATE $\Theta = \theta^{(k)}$ 
    \ENDFOR
      \STATE {\bfseries return} $(\Theta, \mathrm{T},\mathcal{R})$
        
    
    \STATE

    \STATE {\bfseries Subroutine:} $\mathrm{cl\_update\_step}(\theta^{(k)}, y_t)$
    \STATE {\bfseries Input:} Full set of model parameters $\theta^{(k)}$, data mini-batch $y_t$, 
    \STATE {\bfseries Output:} Updated model parameters $\theta^{(k)}$
    \STATE $\theta^{(k)} \leftarrow \theta^{(k)} - \rho_t \nabla L(\theta^{(k)})$,  \ \text{where} \ 
    $L$ is from Eq.\ \eqref{eq:CLloss}
    \STATE {\bfseries return} $\theta^{(k)}$
\end{algorithmic}
\end{small}
\end{algorithm}

Finally, an interesting aspect of using checkpoints for changepoint detection in CL is that 
once the next changepoint $\tau_k$ is detected, and thus we need to instantiate 
a new task parameters $\theta^{p,k+1}$, we can reuse one of the checkpoints to
avoid the full set of model parameters 
$\theta_t \equiv (\theta^s, \{\theta^{p,i}\}_{i=1}^k)$ being contaminated by training updates using data from a new task in iterations $t' \in [\tau_k, t]$, without knowing yet the task change.
Specifically, we can  
re-set this full parameter vector to the nearest checkpoint that exists on  
the left of the changepoint location
$\tau_k$.
This allows the checkpoint
to act as a \emph{recovery state} that can mitigate forgetting of the current $k$-th
task parameters caused by these extra updates, i.e.\ for
$t' \in [\tau_k, t]$.


\section{Related Work
\label{sec:related}}


Changepoint detection methods 
are categorized into offline and online settings
\citep{Aminikhanghahi:2017,truong2018selective}.
Offline algorithms  
such as the recent linear time dynamic programming algorithms  \citep{pruningfearnhead,Maidstone2017} operate similarly to the Viterbi algorithm in hidden Markov models \citep{Bishop:2006}, where they  
need to observe the full data sequence in order to retrospectively identify multiple changepoints. 
In contrast, in  online changepoint detection 
the aim is to detect a change as soon as it occurs 
while data arrive online. Online detection 
has a long history in statistical process control 
\citep{page57,Hawkins2003} where typically 
we want to detect a change in a mean 
parameter in time series.  
More recently, Bayesian online changepoint detection methods
have been developed in \citep{Fearnhead2006,Fearnhead2007,Adams07bayesianonline}, that consider 
conjugate exponential family models and online
Bayesian updates. These latter techniques 
can be extended to also allow online point estimation
of some model parameters \citep{Caron2012,Sinam13}, but 
they remain computationally too 
expensive to use in deep learning 
where models consist of neural networks. This is because they 
are based on Bayesian inference procedures that require selecting suitable 
priors over model parameters and they rely on applying
accurate online Bayesian inference 
which is generally intractable, unless the 
model has a simple conjugate form. Also approximate inference can be too costly 
and inaccurate for highly non-linear and parametrized models.  

The method we introduced differs from these previous approaches, since it relies on the idea of a checkpoint 
which allows to detect changepoints by performing 
multi-step ahead predictions. This setup provides
a stream of 1-dimensional numbers with a simple 
distribution on which we can apply standard statistical tools
to detect whether there exists an abrupt change 
in a window of these predictions.
The checkpoint is updated over time by tracking slowly
(within a distance $T$) the actual model, which can 
improve the discriminative power overtime 
as the task persists and the checkpoint becomes more specialized
to the task distribution. The method can be considered 
as a combination of offline and online detection \citep{Aminikhanghahi:2017} since, while model parameter learning is online, each testing 
with a checkpoint involves an offline subroutine;
see Algorithm \ref{alg:checkpoint_detect_subroutine}. 


More distantly related work is from the recent continual learning literature such as
the so-called task-free or task-agnostic methods
 \citep{Aljundi18a,Aljundi18b,kaplanis2018continual,Zeno2018TaskAC,DushyantRao2019} that 
 learn without knowing or assuming task boundaries. 
 However, the objective there is typically not to explicitly detect changepoints, but instead to maintain an overall good predictive
 performance, by avoiding 
 catastrophic forgetting of the neural network model. 
 In contrast, our method aims to explicitly
 detect abrupt changes in arbitrary online learning 
 systems, either traditional few-parameter models  
 or neural networks used in continual learning. 
 As we discussed in Section \ref{sec:continual} and will 
 demonstrate next  in Section 
 \ref{sec:experiments} our algorithm can be 
 combined with existent continual learning 
 methods and enhance them with the ability of 
 changepoint detection.

\section{Experiments 
\label{sec:experiments}}

\subsection{Time Series Example} 
\label{sec:illustrative_example}

Fig.\ \ref{fig:toyregression1} shows online changepoint detection on an artificial time series dataset, a small snapshot of which was used  
in the illustrative example in Fig.\  \ref{fig:illustration}(b). 
The task is to track a data stream of 1-dimensional noisy observation (each $y_t$ is a scalar value) with abrupt changes in the mean.
The model has a single parameter $\theta$: a moving average that is updated as data arrive online, where the underlying loss 
is $0.5 (y_t -\theta)^2$ (which up to a constant is
the negative log-likelihood of a normal distribution with a fixed variance).
Fig.\ \ref{fig:toyregression1} shows 
changepoint detection achieved by the proposed algorithm. The panel on the top row 
shows the data, the moving average parameter
and the testing windows together with the corresponding checkpoints that lead to
all seven detections.
The panel on the bottom row 
shows the GLR statistics, $- 2 \log \Lambda_{\tau}$, 
computed through time which clearly obtains maximal values at the changepoint locations. The window size was
$T=50$, $\alpha = \lfloor T/4 \rfloor=12$
and $\delta=10^{-3}$.  
All the changepoints are detected by our algorithm without a false alarm. Every changepoint corresponds to
a clear spike in the GLR statistics significantly higher than the normal range of values.

We also study the impact of a sub-optimal choice of learning rate of the tracking model 
to our changepoint detection algorithm in Fig.\ \ref{fig:toyregression1_effectlr}.
As discussed in Section \ref{sec:hyper-parametres}, as long as the score function, 
negative log-likelihood in this example, 
is able to differentiate between different tasks 
the changepoints can still be robustly detected regardless of whether the model is under-fit or over-fit to the data. We should point out, however, that in more complex models 
might not be feasible 
to come up with such discriminate scores. For example, a neural network with random weights most likely will not be so discriminative, i.e.\ it shall provide similar 
(random) predictions for data coming from different tasks. Nevertheless, with reasonable training the neural network can become more specialized to a certain task and provide predictions that can significantly differ from those associated with data from other tasks, that the network hasn't trained on. Therefore, in more complex models the learning rate needs to be chosen carefully to allow quick adaptation to the task data so that the score function, computed under checkpoints, can become more discriminative of task changes.



\begin{figure*}[!htb]
\centering
\begin{tabular}{c}
{\includegraphics[scale=0.32]  
{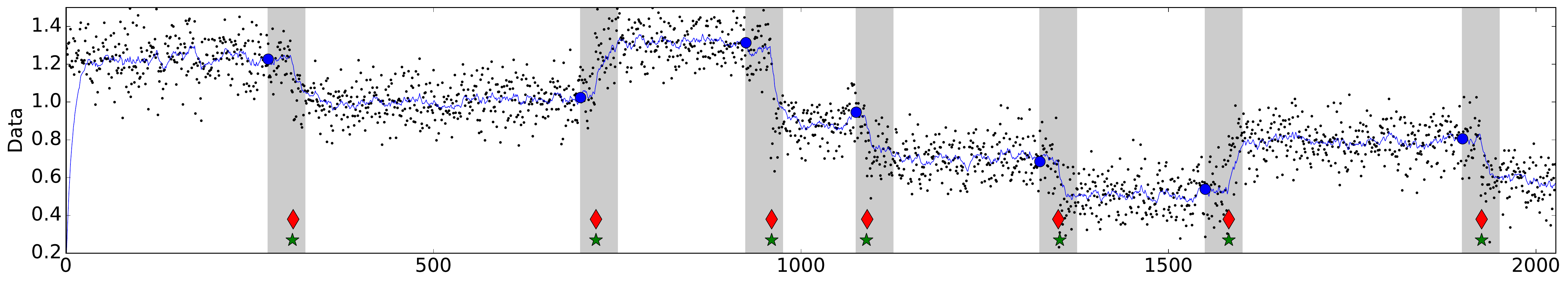}} \\
{\includegraphics[scale=0.32]  
{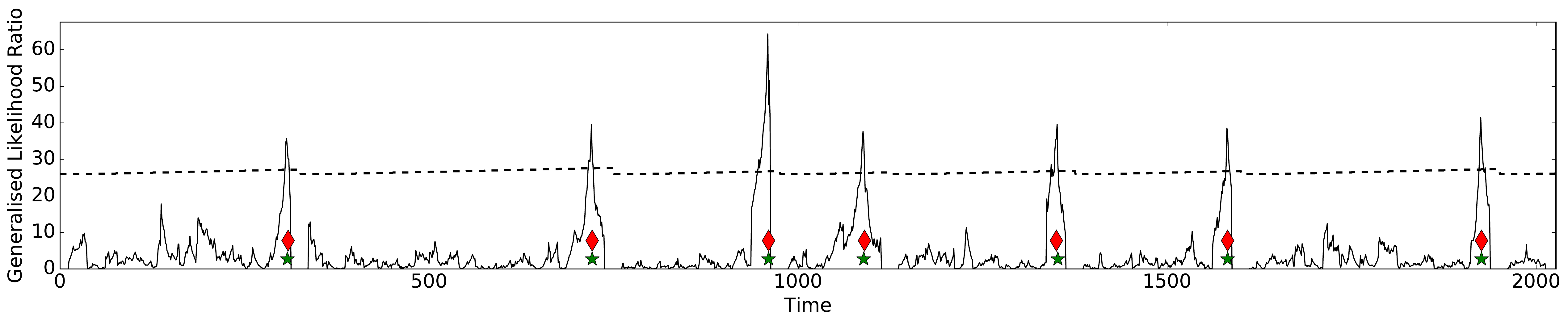}}
\end{tabular}
\caption{Changepoint detection
in an 1-D 
time series. The moving average parameter $\theta$ is initialized at $0$ and 
updated in each step with the gradient update $\theta \leftarrow \theta + 0.1 (y_t - \theta)$, where $\rho=0.1$ is the learning rate. 
({\bf Top}) Data (black dots), 
$\theta$ (blue line) and the detected changepoints, shown as green stars, while the red diamonds are the ground-truth values. 
For each detection all data used in the 
corresponding testing window are highlighted by the shaded areas. The blue dots on the left-borders of these areas are checkpoints.  
({\bf Bottom}) The GLR test values, $- 2 \log \Lambda_\tau $  (solid back line) and the detection threshold, i.e.\ 
$h(\delta_i) = \text{quantile}(1-\delta_i)$ (dotted line), where the latter increases with any new test and resets to its initial value after a detection.} 
\label{fig:toyregression1}
\end{figure*}

 \begin{figure*}[!htb]
\centering
\begin{tabular}{c}
{\includegraphics[scale=0.32]  
{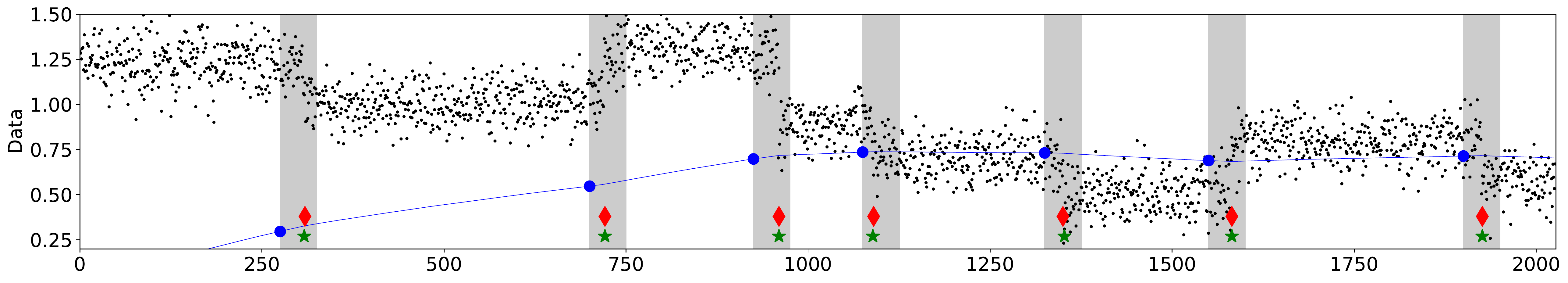}} \\
{\includegraphics[scale=0.32]  
{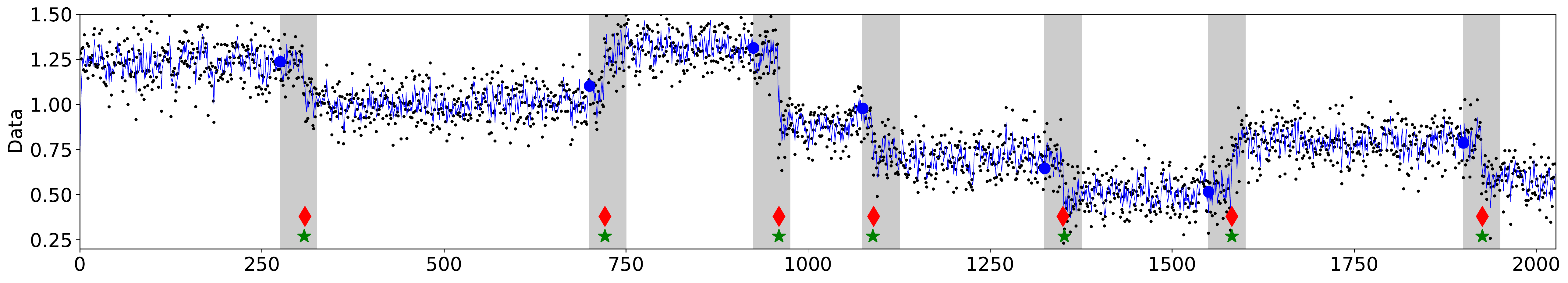}}
\end{tabular}
\caption{Effect of misspecified learning rate. ({\bf Top}) Detection under a very small learning rate having value $\rho=0.001$ that leads to under-fitting of the data. ({\bf Bottom}) Detection under a very large learning with value $\rho=0.5$ that leads to over-fitting. Despite that both rather bad choices of the learning rate result in poor fits to the data, changepoint detection remains accurate with no errors.} 
\label{fig:toyregression1_effectlr}
\end{figure*}

\begin{figure*}[!htb]
\centering
\begin{tabular}{cc}
{\includegraphics[scale=0.19]  
{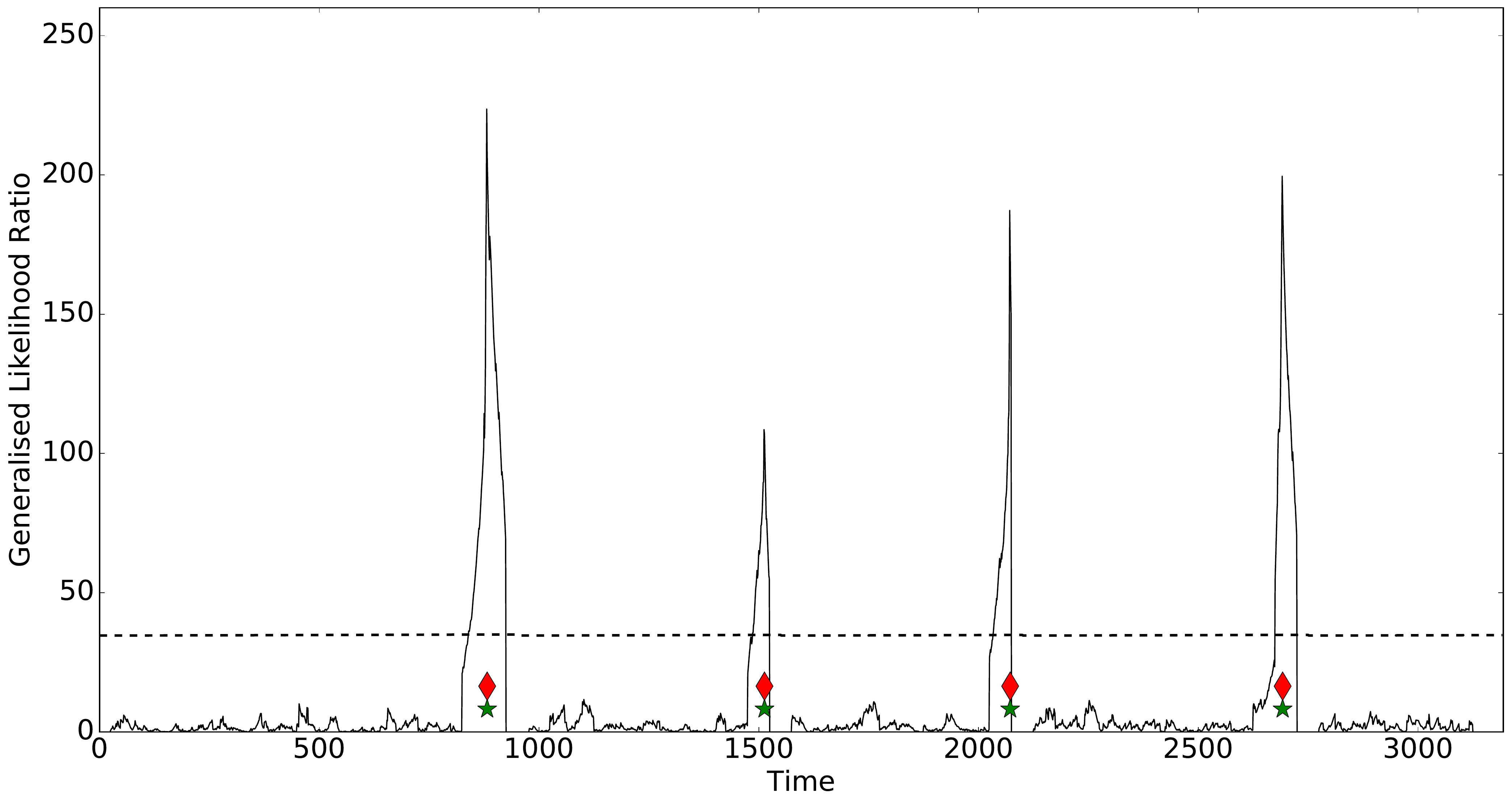}} &
{\includegraphics[scale=0.19]  
{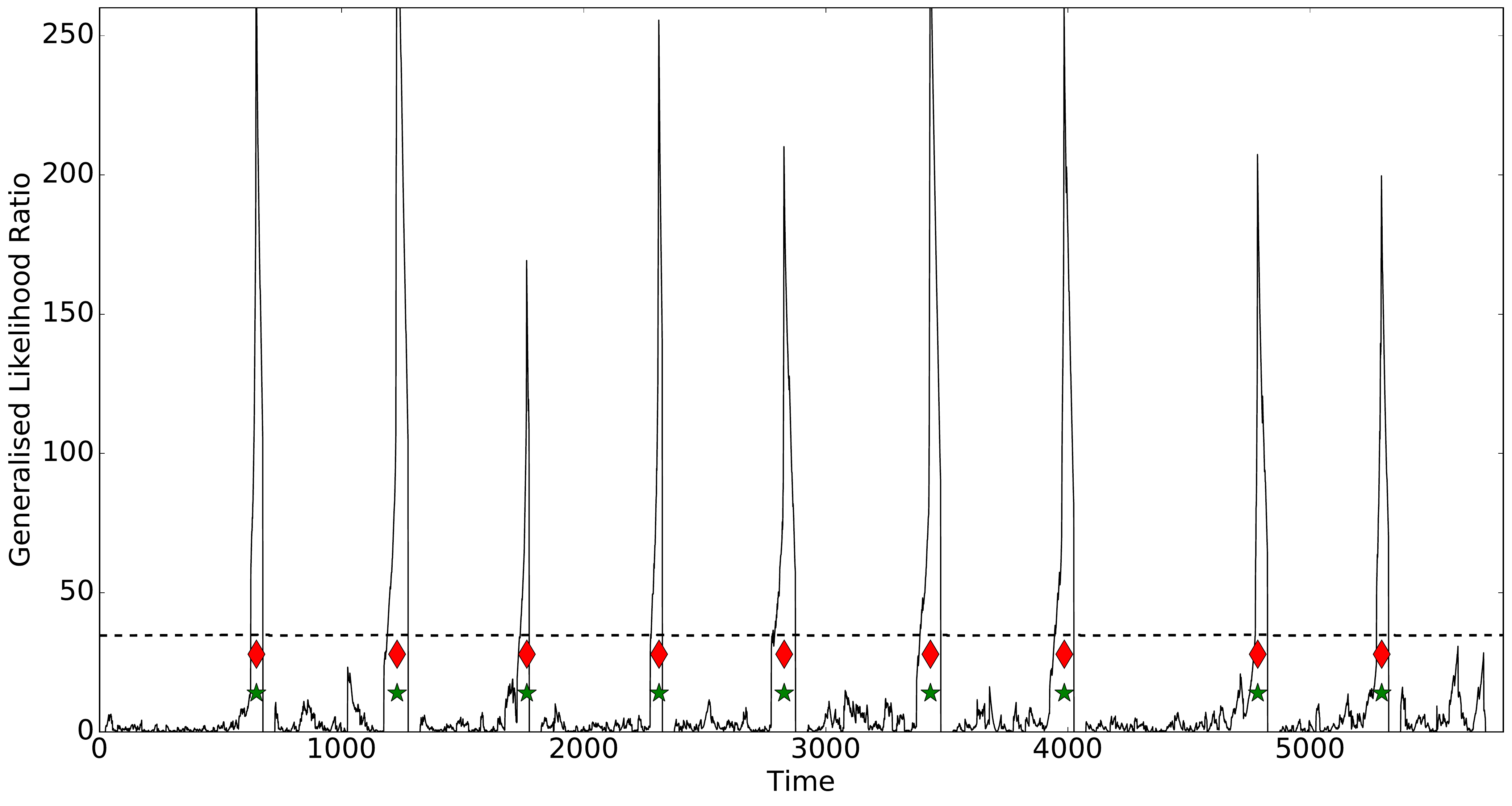}}
 \\
 (a) Split-MNIST   & (b) Permuted-MNIST \\
{\includegraphics[scale=0.19]  
{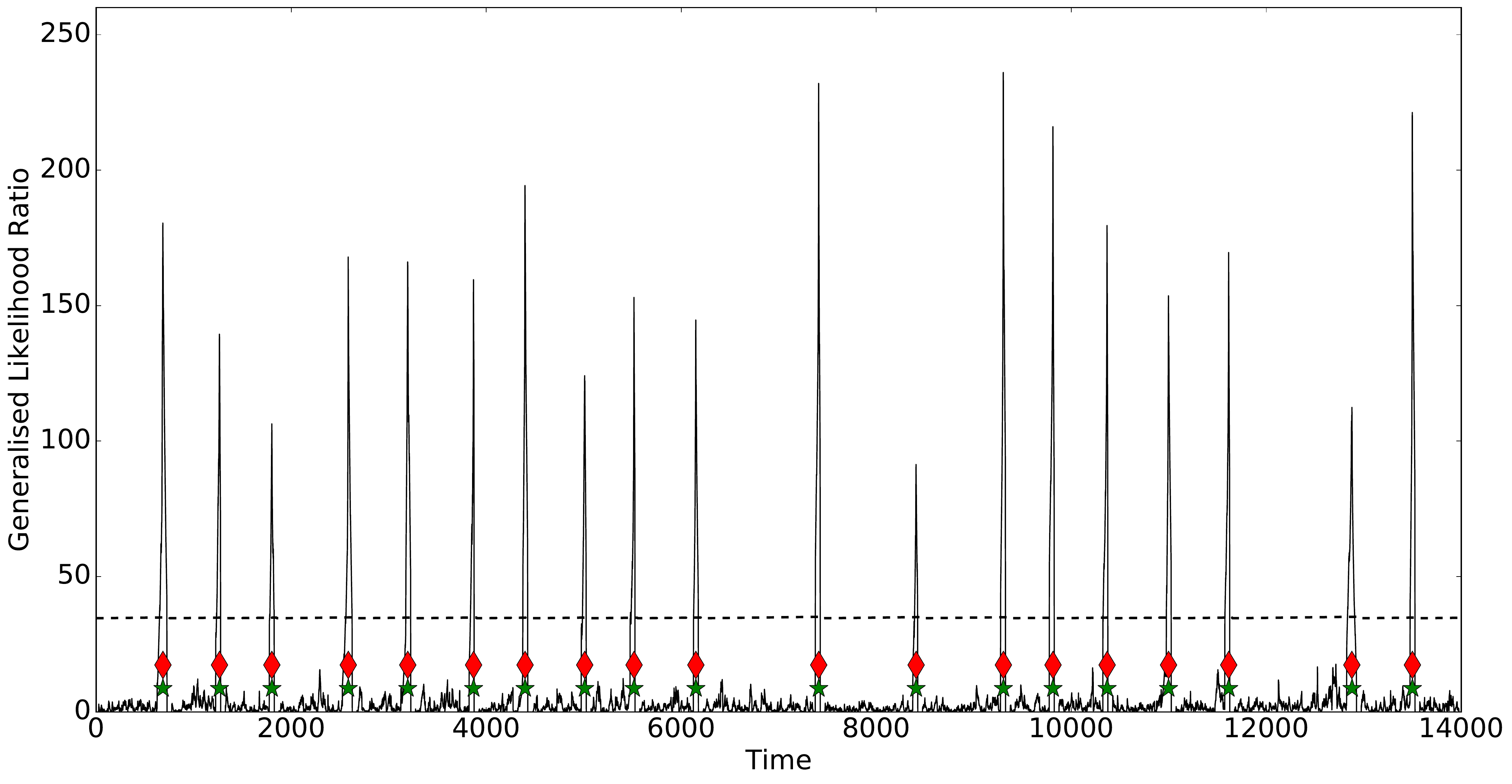}}  &
{\includegraphics[scale=0.19]  
{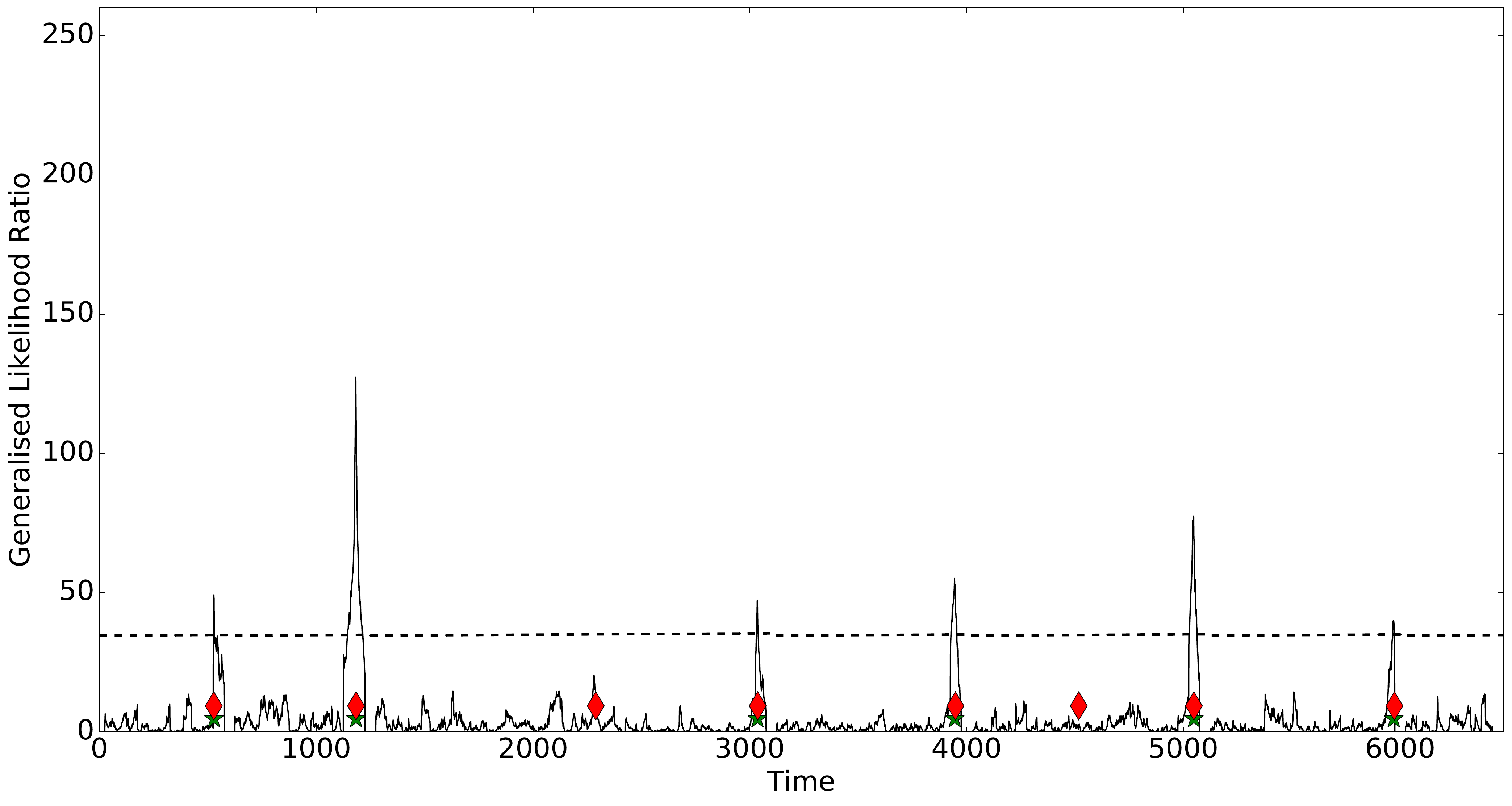}} \\
(c)  CIFAR100  &  (d) Incr-Class-MNIST
\end{tabular}
\caption{GLR test values, $ - 2 \log \Lambda_{\tau}$,  for Split-MNIST, Permuted-MNIST,  CIFAR100 and 
Incr-Class-MNIST.  The mini-batch size for all cases was $10$.} 
\label{fig:mnist}
\end{figure*}

\subsection{Experiments on Continual Learning}

In all CL experiments throughout this section
the proposed Algorithm \ref{alg:checkpoint_detect},
checkpoint-based
changepoint detection (CheckpointCD) is applied, in conjunction with Algorithm \ref{alg:continual_learning}, 
with the following settings:
$$
\delta=10^{-4}, \eta=0.99, 
T=100,\alpha = \lfloor T/4 \rfloor=25.
$$ 
Note that $\eta=0.99$ and $\alpha = \lfloor T/4 \rfloor$ 
are default values, while $\delta$ and
$T$ were specified 
by few preliminary runs on one of the datasets (Split-MNIST). I.e.\ 
the cutoff value of the Type I error was set to $\delta=10^{-4}$ to maximize performance (Jaccard index) on Split-MNIST while for all remaining experiments the same cutoff
is used and is never re-optimized.
The effect of the window size 
$T$ is also analyzed in Fig.\  \ref{fig:windowsizeeffect}.

As a strong baseline for comparison 
we consider Bayesian online changepoint detection
(BayesCD)
by \citet{Adams07bayesianonline}; see
also \citet{Fearnhead2007}.
We define an instant of this method
that is fully applicable to complex models 
such as deep neural networks. This is expressed by 
treating the one-step predictive 
scores $v_t = v(y_t,\theta_{t-1})$
(averaged over the mini-batch at time $t$ so they are close to normality) as sequential observations 
following a univariate Gaussian, $y_t \sim \mathcal{N}(y_t|\mu,\sigma^2)$, where the parameters $(\mu,\sigma^2)$
are task-specific. 
Then, the algorithm detects when $(\mu,\sigma^2)$ undergoes an abrupt change, by performing 
full Bayesian inference and estimating 
recursively the marginal posterior probability 
of each time being a changepoint, i.e.\ 
the so-called task or run length value 
to return to zero value \citep{Adams07bayesianonline}.  
This involves   
placing a conjugate normal inverse-gamma
prior on $(\mu,\sigma^2)$\footnote{The values of the hyperparameters where chosen as $\mu_0=0, \alpha=0.1, \beta=\kappa=1$.}:
$$(\mu,\sigma^2) \sim \mathcal{N}(\mu|\mu_0, \sigma^2/\kappa) IG(\sigma^2|\alpha,\beta),
$$
together with a prior distribution $p(\tau)$ over changepoints, defined through a Hazard function on the run length \citep{Adams07bayesianonline}, that 
models the prior probability of each time being a changepoint. Then Bayesian online learning requires  marginalizing out all unknowns, i.e.\ 
$(\mu,\sigma^2)$ and  the run length.
Because of the conjugate and Markov structure of the model all 
computations are analytic and the marginal posterior 
probability of a changepoint, $p(t=\tau|\text{data})$, across time follows a simple and efficient recursion; 
see \citet{Adams07bayesianonline} for full details. 
To apply the algorithm to CL we need to choose a cut-off threshold for  
$p(t=\tau|\text{data})$ that will
allow to claim a changepoint. We consider a search over different cut-offs and we report 
the best-performing values in Table 1.  As a changepoint prior we consider a constant hazard $H = 1/500$. 

We also included in the comparison 
a simpler (SimpleCD) baseline based on 
purely online 
statistical testing scheme (without requiring storage of checkpoints)  using
the one-step ahead predictive score values $\bfv_t = \{ v(y_t^i, \theta_{t-1})\}_{i=1}^b$, 
where here $\bfv_t$ is a vector of $b$ values and
$b$ is the mini-batch size. Then a standard paired Welch's $t$ 
test $t(\bfv_{t-1}, \bfv_t)$ can be used to detect a changepoint by using a cut-off critical value. In all experiments we considered a set of different critical values and we report the best-performing one 
in Table 1.  

Furthermore, for both BayesCD and SimpleCD algorithms we added the constraint that after a detection 
the algorithm must 
wait $T=100$ time steps to search for a new changepoint, i.e.\ the minimum  distance between two consecutive detections was set to $T$. 
Without this constraint the behaviour of these algorithms can become very noisy resulting in many false positive detections around a previous detected changepoint. Note that this $T$ minimum distance constraint is by definition satisfied by CheckpointCD, as shown in Algorithm \ref{alg:checkpoint_detect}, where $T$ is the window size hyperparameter.

\begin{table*}
\caption{Average Jaccard index scores, with one standard deviations, and tolerance 5
on all CL changepoint detection tasks. The 
numbers inside brackets for the BayesCD 
method indicate different cut-offs 
in the changepoint posterior probability $p(t=\tau|\text{data})$.}
\label{table:allresutls}
\begin{small}
\begin{center}
\begin{tabular}{llcccccr}
\hline\noalign{\smallskip}
Dataset & method & batch size=10 & batch size=20 & batch size=50 &  batch size=100\\
\noalign{\smallskip}\hline\noalign{\smallskip}
Split-MNIST & CheckpointCD  & ${\bf 1.00} \pm 0.00$ & ${\bf 1.00} \pm 0.00$ & ${\bf 1.00} \pm 0.00$ & ${\bf 1.00} \pm 0.00$ \\
& BayesCD(0.3) & $0.38 \pm 0.12$ & $0.71 \pm 0.12$ & $0.93 \pm 0.15$ & ${\bf 1.00} \pm 0.00$ \\
& BayesCD(0.4) & $0.48 \pm 0.16$ & $0.89 \pm 0.14$ &  ${\bf 1.00} \pm 0.00$ & ${\bf 1.00} \pm 0.00$ \\
& BayesCD(0.5) & $0.58 \pm 0.16$ & $0.91 \pm 0.11$ & ${\bf 1.00} \pm 0.00$ & $0.97 \pm 0.07$ \\
& BayesCD(0.6) & $0.71 \pm 0.18$ & $0.93 \pm 0.11$ & $0.97 \pm 0.07$ & $0.95 \pm 0.10$ \\
& SimpleCD & $0.35 \pm 0.21$ & $0.82 \pm 0.13$ & $0.98 \pm 0.06$ & $0.92 \pm 0.13$ \\
\noalign{\smallskip}\hline\noalign{\smallskip}
Permuted-MNIST & CheckpointCD  & ${\bf 0.77} \pm 0.13$ & ${\bf 1.00} \pm 0.00$ & $0.98 \pm 0.06$ & ${\bf 1.00} \pm 0.00$ \\
& BayesCD(0.3) & $0.42 \pm 0.09$ & $0.97 \pm 0.05$ & $0.99 \pm 0.03$ & ${\bf 1.00} \pm 0.00$ \\
& BayesCD(0.4) & $0.44 \pm 0.09$ & $0.93 \pm 0.12$ & ${\bf 1.00} \pm 0.00$ & ${\bf 1.00} \pm 0.00$ \\
& BayesCD(0.5) & $0.57 \pm 0.11$ & ${\bf 1.00} \pm 0.00$ & ${\bf 1.00} \pm 0.00$ & ${\bf 1.00} \pm 0.00$ \\
& BayesCD(0.6) & $0.66 \pm 0.13$ & $0.99 \pm 0.03$ & ${\bf 1.00} \pm 0.00$ & ${\bf 1.00} \pm 0.00$ \\
& SimpleCD & $0.30 \pm 0.08$ & $0.92 \pm 0.06$ & $0.96 \pm 0.07$ & $0.99 \pm 0.03$ \\
 \noalign{\smallskip}\hline\noalign{\smallskip}
Split-CIFAR100 & CheckpointCD   & ${\bf 0.98} \pm 0.04$ & ${\bf 1.00} \pm 0.00$ & $0.99 \pm 0.03$ & ${\bf 1.00} \pm 0.00$ \\
& BayesCD(0.3) & $0.96 \pm 0.04$ & $0.95 \pm 0.07$ & $0.98 \pm 0.02$ & ${\bf 1.00} \pm 0.00$ \\
& BayesCD(0.4) & $0.93 \pm 0.04$ & $0.99 \pm 0.02$ & $0.99 \pm 0.02$ & ${\bf 1.00} \pm 0.00$ \\
& BayesCD(0.5) & $0.93 \pm 0.05$ & ${\bf 1.00} \pm 0.00$ & ${\bf 1.00} \pm 0.00$ & ${\bf 1.00} \pm 0.00$ \\
& BayesCD(0.6) & $0.85 \pm 0.08$ & $0.99 \pm 0.02$ & ${\bf 1.00} \pm 0.00$ & ${\bf 1.00} \pm 0.00$ \\
& SimpleCD & $0.14 \pm 0.08$ & $0.76 \pm 0.06$ & $0.99 \pm 0.02$ & $0.95 \pm 0.04$ \\
\noalign{\smallskip}\hline\noalign{\smallskip}
Incr-Class-MNIST & CheckpointCD   & ${\bf 0.72} \pm 0.14$ & ${\bf 0.96} \pm 0.06$ & ${\bf 0.96} \pm 0.09$ & ${\bf 0.98} \pm 0.07$ \\
& BayesCD(0.3) & $0.41 \pm 0.09$ & $0.64 \pm 0.15$ & $0.85 \pm 0.10$ & $0.81 \pm 0.07$ \\
& BayesCD(0.4) & $0.41 \pm 0.12$ & $0.75 \pm 0.11$ & $0.82 \pm 0.10$ & $0.85 \pm 0.09$ \\
& BayesCD(0.5) & $0.45 \pm 0.13$ & $0.68 \pm 0.12$ & $0.74 \pm 0.08$ & $0.86 \pm 0.07$ \\
& BayesCD(0.6) & $0.43 \pm 0.11$ & $0.60 \pm 0.17$ & $0.81 \pm 0.12$ & $0.81 \pm 0.06$ \\
& SimpleCD & $0.04 \pm 0.05$ & $0.09 \pm 0.08$ & $0.58 \pm 0.19$ & $0.89 \pm 0.09$ \\
\noalign{\smallskip}\hline\noalign{\smallskip}
Incr-Class-CIFAR44 & CheckpointCD   & ${\bf 0.24} \pm 0.05$ & ${\bf 0.80} \pm 0.06$ & ${\bf 0.98} \pm 0.03$ & ${\bf 0.98} \pm 0.05$ \\
& BayesCD(0.3) & $0.07 \pm 0.04$ & $0.43 \pm 0.04$ & $0.81 \pm 0.05$ & $0.96 \pm 0.03$ \\
& BayesCD(0.4) & $0.04 \pm 0.02$ & $0.36 \pm 0.09$ & $0.76 \pm 0.04$ & $0.92 \pm 0.03$ \\
& BayesCD(0.5) & $0.02 \pm 0.02$ & $0.26 \pm 0.06$ & $0.71 \pm 0.06$ & $0.88 \pm 0.04$ \\
& BayesCD(0.6) & $0.02 \pm 0.02$ & $0.20 \pm 0.05$ & $0.64 \pm 0.05$ & $0.86 \pm 0.05$ \\
& SimpleCD & $0.00 \pm 0.01$ & $0.01 \pm 0.01$ & $0.12 \pm 0.04$ & $0.74 \pm 0.05$ \\
\noalign{\smallskip}\hline
\end{tabular}
\end{center}
\end{small}
\end{table*}

\subsubsection{Datasets, CL tasks  and results}

We first
applied the algorithms to three standard CL 
 classification benchmarks: Split-MNIST, 
 Permuted-MNIST and Split-CIFAR100. 
Split-MNIST, introduced by \citet{zenke2017continual}, assumes that five binary classification tasks are constructed from the original MNIST \citep{lecun-mnisthandwrittendigit-2010} hand-written digit classes and they are sequentially  presented to the algorithm in the following order: 0/1, 2/3, 4/5, 6/7, 8/9. Each task is a binary classification problem 
so that any mini-batch $y_t = \{x_t^i, c_t^i\}_{i=1}^b$ is such that each $c_t^i \in \{0,1\}$, i.e.\ the task identity cannot be revealed by inspecting these binary labels. In Permuted-MNIST \citep{goodfellow2013empirical,kirkpatrick2017overcoming}, each task is a variant of the initial $10$-class MNIST classification problem where all input pixels have undergone a fixed (random) permutation. A sequence of 10 tasks is considered so that each task is  a $10$-class classification problem. For the Split-CIFAR100 we assume a sequence
of $20$ tasks of $5$ classes each from the
initial CIFAR100 dataset that contains images of $100$ visual categories, indexed as $1,2,\ldots,100$.  We follow \citet{lopez2017gradient}, 
so that the first task contains the classes $(1,2,3,4,5)$ the second $(6,7,8,9,10)$ and etc.

\begin{figure}[!htb]
\centering
\begin{tabular}{c}
{\includegraphics[scale=0.5]  
{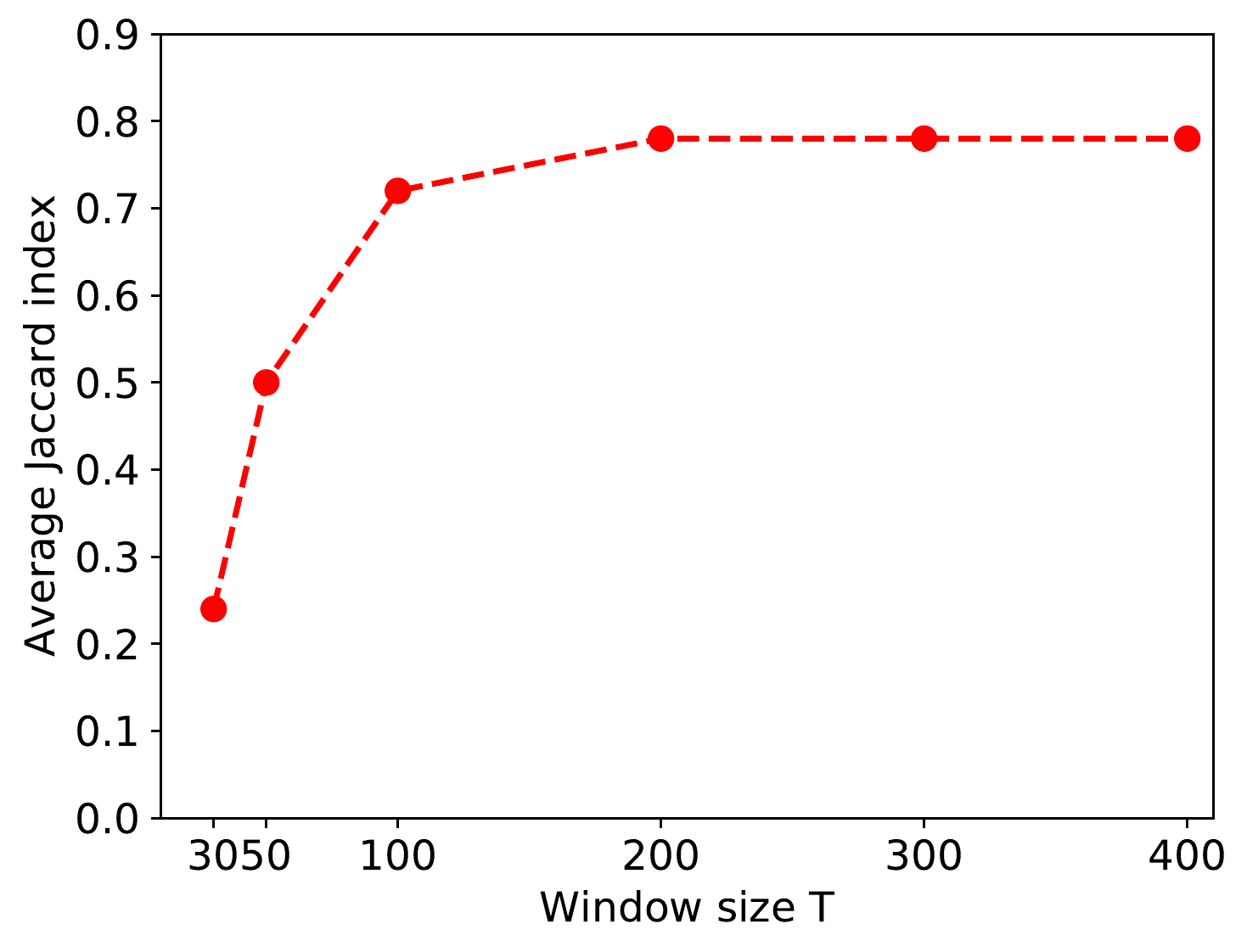}}
\end{tabular}
\caption{Jaccard index averaged over 100 repeats in Incr-Class-MNIST for varying window sizes $T$ and fixed minibatch size 10.} 
\label{fig:windowsizeeffect}
\end{figure}

For Split and Permuted-MNIST we consider a neural network with a shared representation $\phi(x;\theta^s)$
obtained by a fully connected multi-layer perceptron  (MLP) network with two hidden layers of size $100$ and rectified linear units (ReLU) activations. 
For the Split-CIFAR100 we used a much more complex 
residual network architecture \citep{HeZRS16} 
with $18$ layers (ResNet-18), as used by
\citet{lopez2017gradient}.   
We created a simulated experiment where a true task 
changepoint
can occur independently at each time step with some small probability with the additional constraint that we need to observe at least a minimum number of mini-batches from each true task before the next changepoint.\footnote{In all experiments
this number was $500$ and the probability of 
having a changepoint at each step (after the $500$ steps) was $0.005$.} We compare the 
proposed CheckpointCD method with BayesCD and 
SimpleCD under four different values 
of the mini-batch size $b$: 10, 20, 50, 100. 
Training of each CL model was based on Algorithm 
\ref{alg:continual_learning}, modified accordingly 
for BayesCD and SimpleCD so that to apply their respective 
changepoint detection subroutine instead of Algorithm
\ref{alg:checkpoint_detect} used by CheckpointCD. 
The learning rate sequence $\rho_t$ in stochastic gradient optimization of the objective  in Eq.\
\eqref{eq:CLloss} was set in a dimension-wise manner 
by using the Adam optimizer \citep{kingma2014method} which the standard approach for training neural networks; see Appendix for further details.

To measure performance we used the intersection over union score 
(also called Jaccard index) defined as the number of correctly detected changepoints divided by the union of the detected and the true changepoints 
$$
\text{Jaccard index} = \frac{| True \cap Detected|}{| True \cup Detected|} \in  [0,1],
$$
where the larger the score the better.
Note that the Jaccard  index is the hardest among 
other related scores such as recall, 
precision and F1, which are softer/upper bounds (i.e.\ closer to 1) than Jaccard index. For completeness full tables with precision and recall score values are given in the Appendix. 
When computing Jaccard index we also 
allow some tolerance when declaring that a pair 
$(\tau^{tr},\tau^{det})$ of a true and detected locations correspond to the same true changepoint $\tau^{tr}$. 
A tolerance equal to $5$ time steps
distance was used which means 
that only when $|\tau^{tr} - \tau^{det}| \leq 5$ the detected $\tau^{det}$ is considered correct. 

Furthermore, in order to create harder changepoint 
detection problems, we consider two 
class-incremental variants of MNIST and
and CIFAR so that each task differs from 
the previous task by changing only a single class
and without affecting the labeling of the remaining  
classes. This creates the Incr-Class-MNIST 
with 9 tasks: $0/1, 2/1,  2/3, \ldots, 8/7, 8/9$.
To speed up the experiment in CIFAR 
we consider only the first $44$ classes
and create a very challenging changepoint detection problem, Incr-Class-CIFAR44, with $40$ tasks: (1,2,3,4,5), 
(6,2,3,4,5), (6,7,3,4,5) and etc.  

Table \ref{table:allresutls} reports all results
obtained by $10$ random repetitions   
of the experiments. 
The table shows that the proposed algorithm 
is consistently better than the other methods 
and it provides accurate changepoint 
detection even with mini-batch size 
as small as $20$. Notice also that, as expected,
all methods improve as the mini-batch size 
increases.  

Fig.\ \ref{fig:mnist} visualizes the 
GLR values, $- 2\log \Lambda_{\tau}$, in some of the  
runs with Split-MNIST, Permuted-MNIST, CIFAR100
and Incr-Class-MNIST.
Similarly, to Fig.\ \ref{fig:toyregression1}  most  
changepoints are detected by our algorithm and every changepoint corresponds to
a clear spike in the GLR statistics. Note 
that the plots in Fig.\ \ref{fig:mnist} are obtained
for the most difficult case where the data mini-batch 
size when fitting the CL model is $10$, while for larger mini-batch sizes the detection is more robust and the spikes of the GLR statitics become sharper.

Finally, Fig.\ \ref{fig:windowsizeeffect} studies 
the effect of the window size $T$ in changepoint detection performance, which shows that too small value of $T$ could decrease the performance presumably due to very small sample size 
when performing each test. This corroborate our discussion in Section \ref{sec:hyper-parametres} that a large value of $T$ increases the power of hypothesis testing, although it should not be larger than the minimum length of a task from our prior knowledge to avoid including multiple changepoints in the same testing window.

\section{Discussion}

We have introduced an algorithm for online changepoint
detection that can be easily combined with online 
learning of complex non-linear models such as neural networks. We have demonstrated the effectiveness of our method in challenging continual learning tasks
for automatically detecting the task changepoints. 
The use of checkpoints allowed us to define a sequential hypothesis testing procedure to control a
predetermined Type I error upper-bound, and evaluate empirically the overall performance of both Type I and II error using Jaccard index and or other metrics.

The simplicity of checkpoints means that practitioners 
can use them for changepoint detection without having to modify their preferred way of estimating or fitting models to data. For instance,  in deep learning \citep{lecun2015deeplearning} the dominant 
approach to model fitting is based on point parameter 
estimation with stochastic gradient descent (SGD), where the model 
is typically a neural network. As seen in this paper 
this can be easily combined with checkpoints to detect changepoints, without having to modify this standard SGD model fitting procedure.
Similarly, checkpoints could be also combined with other ways of fitting models to data, e.g.\ Bayesian approaches, since the 
essence of the algorithm is a \emph{cached model representation} 
(not necessarily a point parametric estimate)
that together with a prediction score can detect changes. For instance, if we follow a Bayesian model estimation approach,  online learning will require updating a posterior probability distribution $p_t(\theta)$ through time. Then, a checkpoint becomes an early version of this posterior distribution, i.e.\ $p_{t-T}(\theta)$, while the predictive score will be obtained by averaging some function under this checkpoint posterior. In this setting, the use of the algorithm remains the same and the only thing we need to modify, to accommodate this Bayesian way of model fitting, is to change the online model update rule (i.e.\ the line $\theta_t = \mathrm{update\_step}(\theta_{t-1}, y_t)$ in Algorithm \ref{alg:checkpoint_detect}) together with the definition  
of the score function $v(\cdot)$, where the latter should correspond now to a Bayesian predictive score.
While Bayesian model fitting is
very difficult for complex models, such as neural networks, it is certainly feasible for simple conjugate Bayesian models
where we could apply the checkpoint method as outlined above. We leave the experimentation with this more Bayesian way of using checkpoints as a future work.   

Finally, another topic for future research is to consider checkpoints to detect changes at different time scales, 
such as long-term and short-term changes. 


%
%

\bibliographystyle{spbasic}      

\bibliography{refs}


\appendix

%

\section{Quantile of $Z$ statistics in Algorithm \ref{alg:checkpoint_detect_subroutine}}

For every window size $T$, we compute the quantile of the $Z$ statistics (threshold $h(\delta)=\mathrm{quantile}(1 - \delta)$ in Algorithm \ref{alg:checkpoint_detect_subroutine}) numerically with $10^8$ simulations, and fit a linear function for $h(\delta)$. Table \ref{table:z_quantile_30,table:z_quantile_50,table:z_quantile_100,table:z_quantile_200,table:z_quantile_300,table:z_quantile_400} show the computed threshold values as a function of $T$, border size $\alpha$, and eror $\delta$. We also show the fitted line of $h_T(\delta)$ when $\alpha = \lfloor T/4 \rfloor$ as used in the experiments in Figure \ref{fig:error_threshold_line}. We observe that the threshold is close to convergence when $T\geq 100$.

\begin{table}[h!]
\centering
\caption{$Z$-statistics Quantile for T=30.}
\label{table:z_quantile_30}
\begin{small}
\begin{tabular}{ c | c c}
\hline \\ [-1.5ex]
Error $\delta$ & $\alpha=5$ & $\alpha=7$ \\
\hline \\ [-1.5ex]
0.1 & 10.024 & 9.245 \\
0.05 & 11.909 & 11.084 \\
0.01 & 16.089 & 15.173 \\
0.001 & 21.837 & 20.786 \\
0.0001 & 27.464 & 26.273 \\
1e-05 & 33.124 & 31.706 \\
1e-06 & 38.882 & 37.244 \\
\hline 
\end{tabular}
\end{small}
\end{table}

\begin{table}[h!]
\centering
\caption{$Z$-statistics Quantile for T=50.}
\label{table:z_quantile_50}
\begin{small}
\begin{tabular}{ c | c c c}
\hline \\ [-1.5ex]
Error $\delta$ & $\alpha=5$ & $\alpha=10$ & $\alpha=12$ \\
\hline \\ [-1.5ex]
0.1 & 10.661 & 9.318 & 8.948 \\
0.05 & 12.518 & 11.095 & 10.711 \\
0.01 & 16.633 & 15.043 & 14.635 \\
0.001 & 22.283 & 20.454 & 20.018 \\
0.0001 & 27.824 & 25.721 & 25.253 \\
1e-05 & 33.200 & 30.771 & 30.429 \\
1e-06 & 38.841 & 35.969 & 35.716 \\
\hline 
\end{tabular}
\end{small}
\end{table}

\begin{table}[h!]
\centering
\caption{$Z$-statistics Quantile for T=100.}
\label{table:z_quantile_100}
\begin{small}
\begin{tabular}{ c | c c c c c}
\hline \\ [-1.5ex]
Error $\delta$ & $\alpha=5$ & $\alpha=10$ & $\alpha=15$ & $\alpha=20$ & $\alpha=25$ \\
\hline \\ [-1.5ex]
0.1 & 11.270 & 10.303 & 9.726 & 9.244 & 8.789 \\
0.05 & 13.099 & 12.069 & 11.471 & 10.976 & 10.507 \\
0.01 & 17.143 & 15.971 & 15.334 & 14.818 & 14.332 \\
0.001 & 22.705 & 21.315 & 20.628 & 20.088 & 19.578 \\
0.0001 & 28.152 & 26.500 & 25.769 & 25.180 & 24.676 \\
1e-05 & 33.576 & 31.591 & 30.852 & 30.249 & 29.740 \\
1e-06 & 38.957 & 36.549 & 35.712 & 35.165 & 34.720 \\
\hline 
\end{tabular}
\end{small}
\end{table}

\begin{table}[h!]
\centering
\caption{$Z$-statistics Quantile for T=200.}
\label{table:z_quantile_200}
\begin{small}
\begin{tabular}{ c | c c c c}
\hline \\ [-1.5ex]
Error $\delta$ & $\alpha=5$ & $\alpha=20$ & $\alpha=35$ & $\alpha=50$ \\
\hline \\ [-1.5ex]
0.1 & 11.780 & 10.257 & 9.505 & 8.833 \\
0.05 & 13.587 & 11.989 & 11.223 & 10.537 \\
0.01 & 17.588 & 15.817 & 15.028 & 14.323 \\
0.001 & 23.080 & 21.057 & 20.248 & 19.516 \\
0.0001 & 28.482 & 26.173 & 25.320 & 24.611 \\
1e-05 & 33.874 & 31.153 & 30.254 & 29.514 \\
1e-06 & 38.938 & 35.961 & 35.187 & 34.499 \\
\hline 
\end{tabular}
\end{small}
\end{table}

\begin{table}[h!]
\centering
\caption{$Z$-statistics Quantile for T=300.}
\label{table:z_quantile_300}
\begin{small}
\begin{tabular}{ c | c c c c c}
\hline \\ [-1.5ex]
Error $\delta$ & $\alpha=5$ & $\alpha=25$ & $\alpha=45$ & $\alpha=65$ & $\alpha=75$ \\
\hline \\ [-1.5ex]
0.1 & 12.023 & 10.466 & 9.769 & 9.171 & 8.877 \\
0.05 & 13.823 & 12.194 & 11.486 & 10.879 & 10.578 \\
0.01 & 17.799 & 16.007 & 15.285 & 14.666 & 14.357 \\
0.001 & 23.271 & 21.214 & 20.471 & 19.848 & 19.534 \\
0.0001 & 28.667 & 26.270 & 25.498 & 24.868 & 24.546 \\
1e-05 & 34.061 & 31.411 & 30.528 & 29.960 & 29.650 \\
1e-06 & 39.142 & 36.424 & 35.684 & 35.072 & 34.756 \\
\hline 
\end{tabular}
\end{small}
\end{table}

\begin{table}[h!]
\centering
\caption{$Z$-statistics Quantile for T=400.}
\label{table:z_quantile_400}
\begin{small}
\begin{tabular}{ c | c c c c c}
\hline \\ [-1.5ex]
Error $\delta$ & $\alpha=5$ & $\alpha=35$ & $\alpha=65$ & $\alpha=95$ & $\alpha=100$ \\
\hline \\ [-1.5ex]
0.1 & 12.182 & 10.431 & 9.682 & 9.017 & 8.907 \\
0.05 & 13.974 & 12.154 & 11.396 & 10.721 & 10.609 \\
0.01 & 17.935 & 15.957 & 15.187 & 14.496 & 14.382 \\
0.001 & 23.389 & 21.160 & 20.378 & 19.681 & 19.567 \\
0.0001 & 28.729 & 26.232 & 25.423 & 24.730 & 24.621 \\
1e-05 & 34.091 & 31.130 & 30.309 & 29.661 & 29.568 \\
1e-06 & 39.388 & 35.642 & 35.106 & 34.543 & 34.319 \\
\hline 
\end{tabular}
\end{small}
\end{table}

\begin{figure}[tbh!]
\centering
\includegraphics[width=\columnwidth]  
{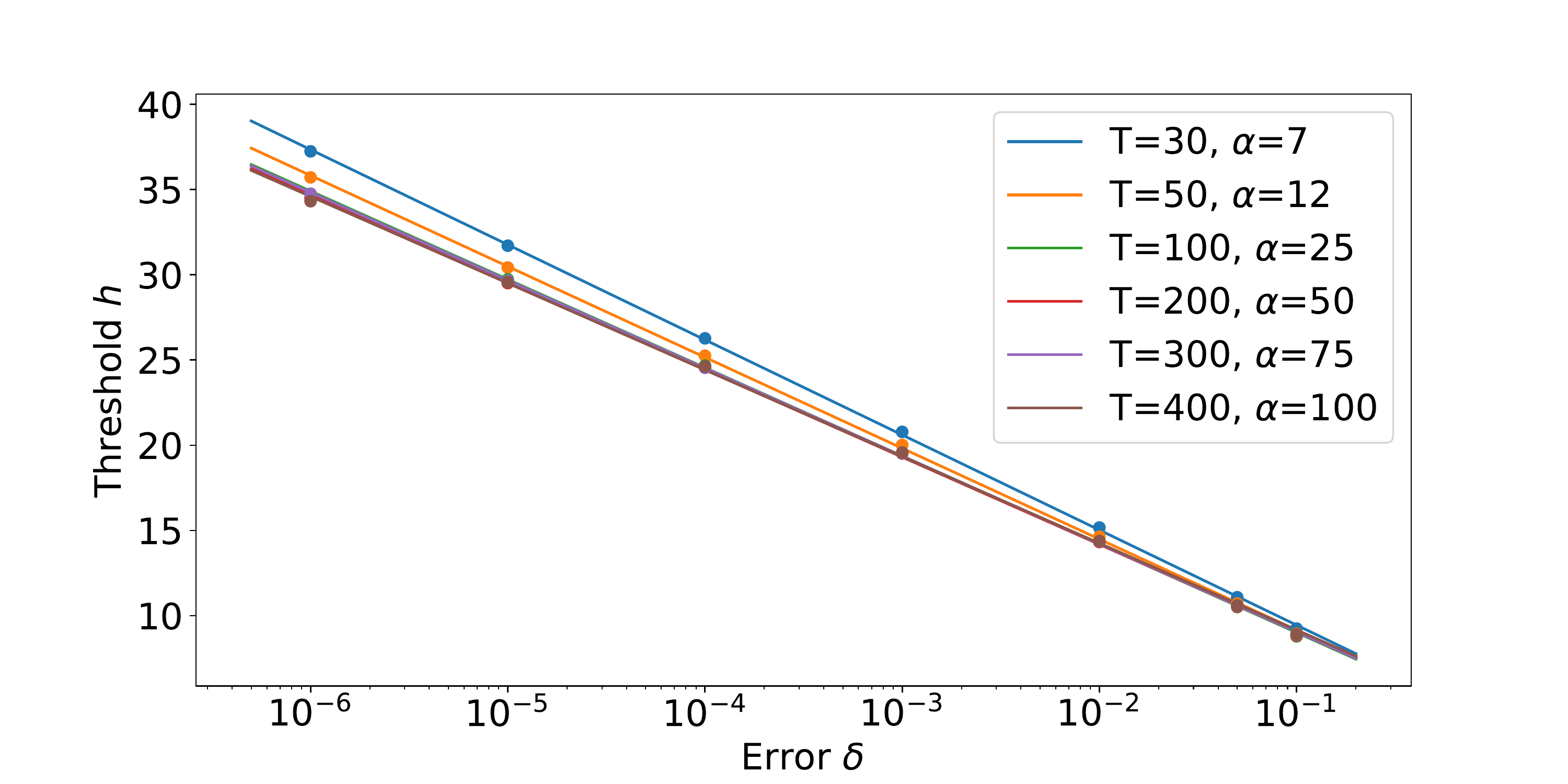}
\caption{Threshold as a function of error $h(\delta)$ for different window size $T$. The border size $\alpha$ is set as $\lfloor T/4 \rfloor$. The circles are computed numerically, and the straight lines are linearly fitted functions.}
\label{fig:error_threshold_line}
\end{figure}

 \section{Further details and results} 
 
 For all CL experiments in Section \ref{sec:experiments} we used the Adam optimizer 
 \citep{kingma2014method} with its
 default parameter settings and with 
 base learning rate value $\alpha = 0.1/b$ 
 where $b$ is the mini-batch size in each experiment. The hyperparameter $\lambda$ in the loss function 
 in Eq.\ \eqref{eq:CLloss} was set to $\lambda=1$ 
 and the size of the replay buffer 
 of each previous task was set to $100$, i.e.\ 
 $|\mathcal{R}_i|=100$. 
 
Table \ref{table:precision} provides the precision scores
and Table \ref{table:recall} the recalls for all 
algorithms applied to the CL benchmarks.  
 
\begin{table*}[t]
\caption{Average precision, with one-standard deviations, and tolerance 5
on all CL changepoint detection tasks. The 
numbers inside brackets for the BayesCD 
method indicate different cut-offs 
in the changepoint posterior probability $p(t=\tau|\text{data})$.}
\label{table:precision}
\begin{small}
\begin{center}
\begin{tabular}{llcccccr}
\hline\noalign{\smallskip}
Dataset & method & batch size=10 & batch size=20 & batch size=50 &  batch size=100\\
\noalign{\smallskip}\hline\noalign{\smallskip}
Split-MNIST & CheckpointCD & $1.00 \pm 0.00$ & $1.00 \pm 0.00$ & $1.00 \pm 0.00$ & $1.00 \pm 0.00$ \\
& BayesCD(0.3) & $0.41 \pm 0.12$ & $0.72 \pm 0.11$ & $0.94 \pm 0.12$ & $1.00 \pm 0.00$ \\
& BayesCD(0.4) & $0.51 \pm 0.15$ & $0.89 \pm 0.14$ & $1.00 \pm 0.00$ & $1.00 \pm 0.00$ \\
& BayesCD(0.5) & $0.63 \pm 0.14$ & $0.96 \pm 0.08$ & $1.00 \pm 0.00$ & $1.00 \pm 0.00$ \\
& BayesCD(0.6) & $0.88 \pm 0.15$ & $1.00 \pm 0.00$ & $1.00 \pm 0.00$ & $1.00 \pm 0.00$ \\
& SimpleCD & $0.62 \pm 0.31$ & $0.96 \pm 0.09$ & $0.98 \pm 0.06$ & $0.94 \pm 0.10$ \\
\noalign{\smallskip}\hline\noalign{\smallskip}
Permuted-MNIST &
CheckpointCD & $0.79 \pm 0.12$ & $1.00 \pm 0.00$ & $0.99 \pm 0.03$ & $1.00 \pm 0.00$ \\
& BayesCD(0.3) & $0.44 \pm 0.09$ & $0.97 \pm 0.05$ & $0.99 \pm 0.03$ & $1.00 \pm 0.00$ \\
& BayesCD(0.4) & $0.45 \pm 0.09$ & $0.94 \pm 0.10$ & $1.00 \pm 0.00$ & $1.00 \pm 0.00$ \\
& BayesCD(0.5) & $0.59 \pm 0.12$ & $1.00 \pm 0.00$ & $1.00 \pm 0.00$ & $1.00 \pm 0.00$ \\
& BayesCD(0.6) & $0.66 \pm 0.13$ & $0.99 \pm 0.03$ & $1.00 \pm 0.00$ & $1.00 \pm 0.00$ \\
& SimpleCD & $0.77 \pm 0.18$ & $0.94 \pm 0.05$ & $0.97 \pm 0.05$ & $0.99 \pm 0.03$ \\
 \noalign{\smallskip}\hline\noalign{\smallskip}
Split-CIFAR100 &
CheckpointCD & $0.99 \pm 0.02$ & $1.00 \pm 0.00$ & $0.99 \pm 0.02$ & $1.00 \pm 0.00$ \\
& BayesCD(0.3) & $0.98 \pm 0.03$ & $0.95 \pm 0.07$ & $0.98 \pm 0.02$ & $1.00 \pm 0.00$ \\
& BayesCD(0.4) & $0.98 \pm 0.03$ & $0.99 \pm 0.02$ & $0.99 \pm 0.02$ & $1.00 \pm 0.00$ \\
& BayesCD(0.5) & $1.00 \pm 0.00$ & $1.00 \pm 0.00$ & $1.00 \pm 0.00$ & $1.00 \pm 0.00$ \\
& BayesCD(0.6) & $0.99 \pm 0.03$ & $1.00 \pm 0.00$ & $1.00 \pm 0.00$ & $1.00 \pm 0.00$ \\
& SimpleCD & $0.54 \pm 0.27$ & $0.94 \pm 0.07$ & $0.99 \pm 0.02$ & $0.96 \pm 0.04$ \\
\noalign{\smallskip}\hline\noalign{\smallskip}
Incr-Class-MNIST &
CheckpointCD & $0.95 \pm 0.08$ & $1.00 \pm 0.00$ & $0.97 \pm 0.05$ & $0.99 \pm 0.04$ \\
& BayesCD(0.3) & $0.47 \pm 0.08$ & $0.76 \pm 0.13$ & $0.97 \pm 0.07$ & $0.97 \pm 0.05$ \\
& BayesCD(0.4) & $0.52 \pm 0.12$ & $0.89 \pm 0.10$ & $1.00 \pm 0.00$ & $1.00 \pm 0.00$ \\
& BayesCD(0.5) & $0.62 \pm 0.15$ & $0.89 \pm 0.11$ & $0.99 \pm 0.04$ & $1.00 \pm 0.00$ \\
& BayesCD(0.6) & $0.73 \pm 0.16$ & $0.94 \pm 0.13$ & $1.00 \pm 0.00$ & $1.00 \pm 0.00$ \\
& SimpleCD & $0.11 \pm 0.14$ & $0.37 \pm 0.37$ & $0.92 \pm 0.11$ & $0.99 \pm 0.03$ \\
\noalign{\smallskip}\hline\noalign{\smallskip}
Incr-Class-CIFAR44 & 
CheckpointCD & $0.96 \pm 0.05$ & $0.96 \pm 0.02$ & $0.99 \pm 0.02$ & $0.99 \pm 0.02$ \\
& BayesCD(0.3) & $0.71 \pm 0.27$ & $0.89 \pm 0.06$ & $0.97 \pm 0.02$ & $1.00 \pm 0.01$ \\
& BayesCD(0.4) & $0.75 \pm 0.25$ & $0.94 \pm 0.05$ & $0.98 \pm 0.04$ & $0.99 \pm 0.01$ \\
& BayesCD(0.5) & $0.58 \pm 0.44$ & $0.95 \pm 0.10$ & $1.00 \pm 0.01$ & $1.00 \pm 0.00$ \\
& BayesCD(0.6) & $0.60 \pm 0.49$ & $0.98 \pm 0.05$ & $1.00 \pm 0.00$ & $1.00 \pm 0.00$ \\
& SimpleCD & $0.02 \pm 0.05$ & $0.11 \pm 0.17$ & $0.86 \pm 0.18$ & $0.96 \pm 0.03$ \\
\noalign{\smallskip}\hline
\end{tabular}
\end{center}
\end{small}
\end{table*}

\begin{table*}[t]
\caption{Average recall, with one-standard deviations, and tolerance 5
on all CL changepoint detection tasks. The 
numbers inside brackets for the BayesCD 
method indicate different cut-offs 
in the changepoint posterior probability $p(t=\tau|\text{data})$.}
\label{table:recall}
\begin{small}
\begin{center}
\begin{tabular}{llcccccr}
\hline\noalign{\smallskip}
Dataset & method & batch size=10 & batch size=20 & batch size=50 &  batch size=100\\
\noalign{\smallskip}\hline\noalign{\smallskip}
Split-MNIST &
CheckpointCD & $1.00 \pm 0.00$ & $1.00 \pm 0.00$ & $1.00 \pm 0.00$ & $1.00 \pm 0.00$ \\
& BayesCD(0.3) & $0.85 \pm 0.17$ & $0.97 \pm 0.07$ & $0.97 \pm 0.07$ & $1.00 \pm 0.00$ \\
& BayesCD(0.4) & $0.85 \pm 0.17$ & $1.00 \pm 0.00$ & $1.00 \pm 0.00$ & $1.00 \pm 0.00$ \\
& BayesCD(0.5) & $0.85 \pm 0.17$ & $0.95 \pm 0.10$ & $1.00 \pm 0.00$ & $0.97 \pm 0.07$ \\
& BayesCD(0.6) & $0.80 \pm 0.19$ & $0.93 \pm 0.11$ & $0.97 \pm 0.07$ & $0.95 \pm 0.10$ \\
& SimpleCD & $0.42 \pm 0.23$ & $0.85 \pm 0.12$ & $1.00 \pm 0.00$ & $0.97 \pm 0.07$ \\
\noalign{\smallskip}\hline\noalign{\smallskip}
Permuted-MNIST &
CheckpointCD & $0.97 \pm 0.05$ & $1.00 \pm 0.00$ & $0.99 \pm 0.03$ & $1.00 \pm 0.00$ \\
& BayesCD(0.3) & $0.91 \pm 0.07$ & $1.00 \pm 0.00$ & $1.00 \pm 0.00$ & $1.00 \pm 0.00$ \\
& BayesCD(0.4) & $0.91 \pm 0.10$ & $0.98 \pm 0.04$ & $1.00 \pm 0.00$ & $1.00 \pm 0.00$ \\
& BayesCD(0.5) & $0.94 \pm 0.06$ & $1.00 \pm 0.00$ & $1.00 \pm 0.00$ & $1.00 \pm 0.00$ \\
& BayesCD(0.6) & $0.99 \pm 0.03$ & $1.00 \pm 0.00$ & $1.00 \pm 0.00$ & $1.00 \pm 0.00$ \\
& SimpleCD & $0.34 \pm 0.10$ & $0.98 \pm 0.04$ & $0.99 \pm 0.03$ & $1.00 \pm 0.00$ \\
 \noalign{\smallskip}\hline\noalign{\smallskip}
Split-CIFAR100 &
CheckpointCD & $0.99 \pm 0.02$ & $1.00 \pm 0.00$ & $0.99 \pm 0.02$ & $1.00 \pm 0.00$ \\
& BayesCD(0.3) & $0.98 \pm 0.03$ & $0.99 \pm 0.02$ & $1.00 \pm 0.00$ & $1.00 \pm 0.00$ \\
& BayesCD(0.4) & $0.94 \pm 0.03$ & $1.00 \pm 0.00$ & $1.00 \pm 0.00$ & $1.00 \pm 0.00$ \\
& BayesCD(0.5) & $0.93 \pm 0.05$ & $1.00 \pm 0.00$ & $1.00 \pm 0.00$ & $1.00 \pm 0.00$ \\
& BayesCD(0.6) & $0.86 \pm 0.08$ & $0.99 \pm 0.02$ & $1.00 \pm 0.00$ & $1.00 \pm 0.00$ \\
& SimpleCD4.0 & $0.15 \pm 0.09$ & $0.81 \pm 0.05$ & $1.00 \pm 0.00$ & $0.99 \pm 0.02$ \\
\noalign{\smallskip}\hline\noalign{\smallskip}
Incr-Class-MNIST &
CheckpointCD & $0.74 \pm 0.12$ & $0.96 \pm 0.06$ & $0.97 \pm 0.05$ & $0.99 \pm 0.04$ \\
& BayesCD(0.3) & $0.74 \pm 0.16$ & $0.78 \pm 0.12$ & $0.88 \pm 0.08$ & $0.82 \pm 0.06$ \\
& BayesCD(0.4) & $0.64 \pm 0.13$ & $0.82 \pm 0.10$ & $0.82 \pm 0.10$ & $0.85 \pm 0.09$ \\
& BayesCD(0.5) & $0.61 \pm 0.13$ & $0.74 \pm 0.12$ & $0.75 \pm 0.08$ & $0.86 \pm 0.07$ \\
& BayesCD(0.6) & $0.51 \pm 0.13$ & $0.61 \pm 0.15$ & $0.81 \pm 0.12$ & $0.81 \pm 0.06$ \\
& SimpleCD & $0.05 \pm 0.06$ & $0.10 \pm 0.09$ & $0.60 \pm 0.18$ & $0.90 \pm 0.09$ \\
\noalign{\smallskip}\hline\noalign{\smallskip}
Incr-Class-CIFAR44 & 
CheckpointCD & $0.24 \pm 0.05$ & $0.82 \pm 0.05$ & $0.99 \pm 0.02$ & $0.99 \pm 0.02$ \\
& BayesCD(0.3) & $0.07 \pm 0.04$ & $0.45 \pm 0.04$ & $0.83 \pm 0.04$ & $0.96 \pm 0.02$ \\
& BayesCD(0.4) & $0.04 \pm 0.02$ & $0.37 \pm 0.09$ & $0.77 \pm 0.03$ & $0.92 \pm 0.03$ \\
& BayesCD(0.5) & $0.02 \pm 0.02$ & $0.27 \pm 0.06$ & $0.72 \pm 0.06$ & $0.88 \pm 0.04$ \\
& BayesCD(0.6) & $0.02 \pm 0.02$ & $0.20 \pm 0.05$ & $0.63 \pm 0.05$ & $0.86 \pm 0.05$ \\
& SimpleCD & $0.00 \pm 0.01$ & $0.01 \pm 0.01$ & $0.13 \pm 0.05$ & $0.76 \pm 0.04$ \\
\noalign{\smallskip}\hline
\end{tabular}
\end{center}
\end{small}
\end{table*}

\end{document}